\documentclass[11pt]{article}

\date{}
\usepackage{times}

\usepackage[utf8]{inputenc} 
\usepackage[T1]{fontenc}    
\usepackage{hyperref}       
\usepackage{url}            
\usepackage{booktabs}       
\usepackage{amsfonts}       
\usepackage{nicefrac}       
\usepackage{microtype}      
\usepackage{xspace}

\usepackage{comment}

\usepackage{rotating}

\usepackage{amssymb}
\usepackage{autobreak}
\usepackage{mathtools}
\usepackage{xcolor}

\usepackage{algorithm}
\usepackage[noend]{algpseudocode}

\usepackage[numbers, compress]{natbib}

\usepackage{amsfonts}

\usepackage{times}
\usepackage{xcolor}     
\usepackage{caption}
\usepackage{mathtools}
\usepackage{enumitem}
\usepackage{bm}

\usepackage{microtype}
\usepackage{graphicx}

\usepackage{hyperref}

\usepackage{amsmath}
\usepackage{amssymb}
\usepackage{mathtools}
\usepackage{amsthm}

\usepackage{xfrac}

\usepackage[capitalize,noabbrev]{cleveref}
\usepackage{xspace}
\usepackage{multirow}
\usepackage{enumitem}
\usepackage{xfrac}
\newcommand{\method}{CVGP\xspace}
\newcommand{\cvexpr}{\texttt{CVExp}}

\theoremstyle{plain}
\newtheorem{theorem}{Theorem}[section]
\newtheorem{proposition}[theorem]{Proposition}
\newtheorem{lemma}[theorem]{Lemma}

\theoremstyle{definition}
\newtheorem{definition}[theorem]{Definition}

\theoremstyle{remark}

\usepackage{xspace}

\title{Symbolic Regression via Control Variable Genetic Programming}
\author{Nan Jiang, Yexiang Xue\\
Department of Computer Science, Purdue University \\
\texttt{\{jiang631,yexiang\}@purdue.edu}}

\begin{document}

\maketitle

\begin{abstract}

Learning symbolic expressions  directly from experiment data  is a vital step in AI-driven scientific discovery.
Nevertheless, state-of-the-art approaches are limited to learning simple expressions. 
Regressing expressions involving many independent variables still remain out of reach. 
Motivated by the control variable experiments widely utilized in science, we propose \underline{\textbf{C}}ontrol \underline{\textbf{V}}ariable \underline{\textbf{G}}enetic \underline{\textbf{P}}rogramming (\method) for  
 symbolic regression over many independent variables. 
 \method~expedites symbolic expression discovery via customized experiment design, rather than learning from a fixed dataset collected a priori. 
 \method~starts by fitting simple expressions involving a small set of independent variables using genetic programming, under controlled experiments where other variables are held as constants. 
 It then extends expressions learned in previous generations by adding new independent variables, using new control variable experiments in which these variables are allowed to vary. 
Theoretically, we show \method~as an incremental building approach can yield an exponential reduction in the search space when learning a class of expressions. 
Experimentally, \method~outperforms several baselines in learning symbolic expressions involving multiple independent variables. 
\end{abstract}

\section{Introduction}

Discovering scientific laws automatically from experiment data has been a grand goal of Artificial Intelligence (AI). Its success will greatly accelerate the pace of scientific discovery. 
Symbolic regression, \textit{i.e.}, learning symbolic expressions from data, consists of a vital 
step in realizing this grand goal. 
Recently, exciting progress~\cite{doi:10.1126/science.1165893,DBLP:conf/gecco/VirgolinAB19,guimera2020bayesian,DBLP:conf/iclr/PetersenLMSKK21,DBLP:conf/nips/MundhenkLGSFP21,DBLP:conf/iclr/PetersenLMSKK21,DBLP:journals/corr/abs-2205-11107,DBLP:journals/corr/abs-2203-08808,DBLP:conf/gecco/HeLYLW22} has been made in this domain, especially with the aid of deep neural networks. 
Despite great achievements, state-of-the-art approaches are limited to learning relatively 
simple expressions, often involving a small set of variables.
Regressing symbolic expressions involving many independent variables still remains out of reach of current approaches. 
The difficulty mainly lies in the exponentially large search space of symbolic expressions.

Our work attacks this major gap of symbolic regression, leveraging control variable experimentation -- a classic procedure widely implemented in the science 
community~\cite{lehman2004designing,DBLP:books/daglib/0022270}. 
In the analysis of complex scientific phenomena involving many contributing factors, control variable experiments are conducted where a set of factors are held constant (\textit{i.e.}, controlled variables), and the dependence between the output variable and the remaining input variables is studied~\cite{DBLP:journals/ml/Langley88,Kibler1991TheES}.
The result is a reduced-form expression that models the dependence only among the output and the non-controlled variables. 
Once the reduced-form equation is validated, scientists introduce more variables
into play by freeing a few controlled variables in previous experiments. 
The new goal is to extend the previous equation to a general one including the newly introduced variables. 
This process continues until all independent variables are introduced into the model.

Our proposed \underline{\textbf{C}}ontrol \underline{\textbf{V}}ariable \underline{\textbf{G}}enetic \underline{\textbf{P}}rogramming (\method) approach implements the aforementioned scientific discovery process using Genetic Programming (GP) for symbolic regression over many independent variables. 
The key insight of \method~is to learn from \textit{a customized set of control variable experiments}; in other words, the experiment data collection adapts to the learning process. 
This is in contrast to the current learning paradigm of most symbolic regression approaches, where they learn from a fixed dataset collected a priori. 
In \method, first, we hold all independent variables except for one as constants and learn a symbolic expression that maps the single variable to the dependent variable using GP. 
GP maintains a pool of candidate equations and improves the fitness of these equations
via mating, mutating, and selection over several generations. 
Mapping the dependence of one independent variable is easy. Hence GP can usually 
recover the ground-truth reduced-form equation. 
Then, \method~frees one  independent variable at a time. In each iteration, GP is used to modify the equations learned in previous generations to incorporate the new independent variable. 
This step is again conducted via mating, mutating, and selection.
Such a procedure repeats until all the independent variables have been incorporated into the symbolic expression. 
After discovering \method~independently,  the authors learned in private communications a line of research work \cite{DBLP:conf/ijcai/Langley77,DBLP:conf/ijcai/Langley79,DBLP:conf/ijcai/LangleyBS81,king2004functional,king2009autosci,cerrato2023rlsci} that also implemented the human scientific discovery process using AI, pioneered by the BACON systems developed by Langley, P. in 1978-1981~\cite{DBLP:conf/ijcai/Langley77,DBLP:conf/ijcai/Langley79,DBLP:conf/ijcai/LangleyBS81}. 
While BACON's discovery was driven by rule-based engines and our \method~uses modern machine learning approaches such as genetic programming, indeed both approaches share a common vision – the \textit{integration of experiment design and model learning} can further expedite scientific discovery. 

Theoretically, we show \method as an incremental builder can reduce the exponential-sized search space for candidate  expressions into a polynomial one when fitting a class of 
symbolic expressions. 
Experimentally, we show \method~outperforms a number of state-of-the-art approaches on symbolic regression over multiple independent variables. 
Our contributions can be summarized as:
\begin{itemize}[align=left, leftmargin=0pt, labelwidth=0pt, itemindent=!]
    \item We propose \method, an incremental builder for symbolic regression over many independent variables.  \method~fits increasingly more complex equations via conducting control variable experiments with fewer and fewer controlled variables.
    \item Theoretically, we show such an incremental builder as \method~can reduce exponential-sized search spaces for symbolic regression to polynomial ones when searching for a class of symbolic expressions. 
    \item Empirically, we demonstrate \method~outperforms state-of-the-art symbolic regression approaches in discovering multi-variable equations from data. 
\end{itemize}

\section{Preliminaries}

\noindent\textbf{Symbolic Expression.}
A symbolic expression $\phi$ is expressed as variables and constants connected by a set of operators. Variables are allowed to change across different contexts, while constants remain the same. Each operator has a predetermined arity, \textit{i.e.}, the number of operands taken by the operator. 
Each operand of an operator is either a variable, a constant, or a self-contained symbolic expression. 
A symbolic expression can also be drawn as a tree, where variables and constants reside in leaves, and operators reside in inner nodes. See Figure \ref{fig:reduced}(a) for an example. 
In this paper, we deal with symbolic expressions involving  real numbers. The semantic meaning of a symbolic expression follows its standard definition in arithmetics and thus is omitted. 

\noindent\textbf{Symbolic Regression.} 
Given a dataset $\{(\mathbf{x}_i, y_i)\}_{i=1}^n$ and a loss function $\ell(\cdot,\cdot)$, where  $\mathbf{x}_i\in \mathbb{R}^m$ and $y_i\in \mathbb{R}$, the objective of symbolic regression (SR) is to search for the optimal symbolic expression $\phi^*$ within the space of all candidate expressions $\Pi$ that minimizes the average loss: 
\begin{equation}
\phi^*=\arg\min_{\phi\in \Pi}\;\frac{1}{n} \sum_{i=1}^n \ell(\phi(\mathbf{x}_i), y_i),
\end{equation}
in addition to regularization terms. 
Symbolic regression is challenging and is shown to be NP-hard~\cite{journal/tmlr/virgolin2022}, due to the exponentially large space of candidate symbolic expressions. 

\noindent\textbf{Genetic Programming for Symbolic Regression.}
Genetic Programming (GP) has been a popular method to solve symbolic regression. Recently, a few other approaches based on  neural networks surpassed the performance of GP in symbolic regression. We leave the discussions of these methods 
to the related work section. 
The high-level idea of GP  is to maintain a pool of candidate symbolic expressions. In each generation, candidate expressions are \textit{mutated} with probability $P_{mu}$ and \textit{mated} with probability $P_{ma}$. Then in the \textit{selection} step, 
those with the highest fitness scores, measured by how each expression predicts
the output from the input, are selected as the candidates for the next generation, together with a few randomly chosen ones to maintain diversity.
After several generations, expressions with high fitness scores, \textit{i.e.}, those fit data well survive in the pool of candidate  solutions. 
The best expressions found in all generations are recorded as {hall-of-fame} solutions.

\section{Control Variable Genetic Programming}
 In this section, we present our control variable genetic programming algorithm.
 Before we dive into the algorithm description, 
 we first need to study what are the outcomes of a control variable experiment and 
what conclusions we can draw on the symbolic regression 
 expression by observing such outcomes. 

\subsection{Control Variable Experiment}
\begin{figure*}[!t]
    \centering
    \includegraphics[width=\textwidth]{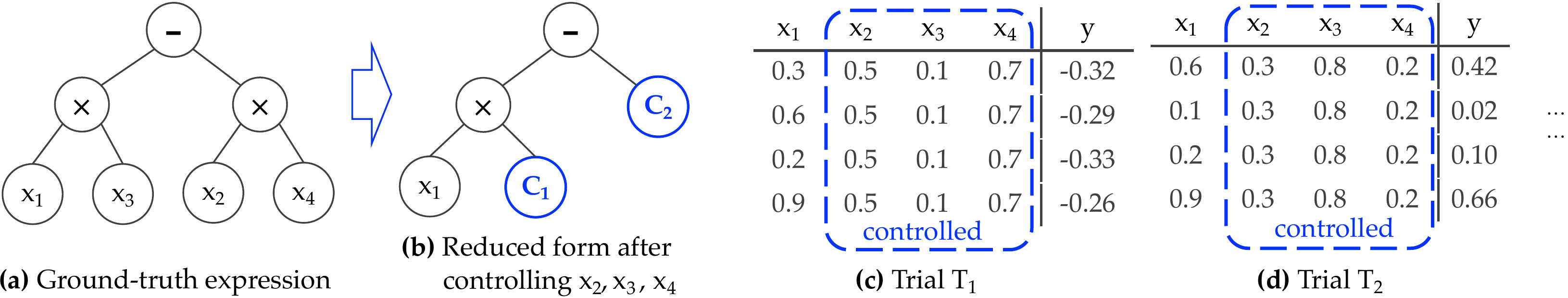}
    \caption{An example of two trials of a control variable experiment. 
    \textbf{(a)} The data of the experiment is generated by the ground-truth 
    expression $\phi=x_1x_3-x_2x_4$. \textbf{(b)} If we control
    $\mathbf{v}_c=\{x_2, x_3, x_4\}$ and only allow $\mathbf{v}_f=\{x_1\}$ to vary, it \textit{looks like}  
     the data are generated from the reduced-form equation $\phi'=C_1x_1-C_2$. \textbf{(c, d)} The generated data in two trials of the control variable experiments.
    The controlled variables are fixed within each trial but vary across trials.}
    \label{fig:reduced}
\end{figure*}

A control variable experiment $\cvexpr(\phi, \mathbf{v}_c, \mathbf{v}_f, $ $\{T_k\}_{k=1}^K)$
consists of the trial symbolic expression $\phi$, 
a set of controlled variables $\mathbf{v}_c$, 
a set of free variables $\mathbf{v}_f$, and $K$ trial experiments $T_1, \ldots, T_K$. 
The expression $\phi$ may have zero or multiple \textit{open constants}. 
The value of an open constant is determined by fitting the equation to the training data.

\smallskip
\noindent\textbf{One Trial in a Control Variable Experiment.} 
A single trial of a control variable experiment $T_k$ fits the symbolic expression $\phi$ with a batch of data. 
To avoid abusing notations, we also use $T_k$ to denote the batch of data.
In the generated data $T_k$, every controlled variable is fixed to the same value while the free variables are set randomly. 
We assume that the values of the dependent variables in a batch are (noisy observations) of the ground-truth expressions with the values of independent variables set in the batch. 
In science, this step is achieved by conducting real-world  experiments, \textit{i.e.}, controlling independent variables, and performing measurements on the dependent variable. 
For example, Fig.~\ref{fig:reduced}(c,d) demonstrates two trials of a control variable experiment in which variable $x_2, x_3, x_4$ are controlled, \textit{i.e.}, $\mathbf{v}_c=\{x_2,x_3,x_4\}$. 
They are fixed to one value in trial $T_1$ (in Fig.~\ref{fig:reduced}(c)) and another value in trial $T_2$ (in Fig.~\ref{fig:reduced}(d)).
$x_1$ is the only free variable, \textit{i.e.}, $\mathbf{v}_f=\{x_1\}$. The value of $x_1$ is varied in trials $T_1, T_2$. 

\smallskip
\noindent\textbf{Reduced-form  Expression in a Control Variable Setting.} 
We assume there is a ground-truth symbolic expression that produces the experiment data. 
In other words, the observed output is the execution of the ground-truth expression from the input, possibly in addition to some noise.
In control variable experiments, 
because the values of controlled variables are 
fixed in each trial, 
what we observe is the ground-truth expression in its \textit{reduced form}, where sub-expressions 
involving only controlled variables are replaced
with constants. 
Fig.~\ref{fig:reduced}(b) provides an example of the reduced form expression. 
Assume the data is
generated from the ground-truth expression in (a): $\phi=x_1 x_3 - x_2 x_4$.
When we control the values of variable in $\mathbf{v}_c=\{x_2, x_3, x_4\}$, the data  \textit{looks like} they are generated from the \textit{reduced} expression: $\phi'=C_1 x_1 - C_2$. 
Here $C_1$ replaces the controlled variable $x_3$, and $C_2$ replaces a sub-expression $x_2 x_4$ in the ground-truth expression. 
We can see both $C_1$ and $C_2$ hold constant values in each trial. 
However, their values vary across trials because the values of controlled variables change.
Back to the example, in trial $T_1$, when $x_2$, $x_3$, and $x_4$ are fixed to 0.5, 0.1, 0.7,  $C_1$ takes the  value of $x_3$, \textit{i.e.}, 0.1, while $C_2$ takes the value of $x_2 x_4$, \textit{i.e.}, 0.35. In trial $T_2$, $C_1=0.8$ and  $C_2=0.06$.

We call constants which represent sub-expressions involving controlled variables in the ground-truth expression \textit{summary constants}, and refer to 
constants in the ground-truth expression \textit{stand-alone constants}. 
For example, $C_1$ and $C_2$ in Fig.~\ref{fig:reduced}(b) are both summary constants. 
Notice that the types of constants are \textit{unknown} in the process of fitting an expression to control variable experiment data. 
However, the best-fitted values of these constants across several trials reveal important information:
a constant is probably a summary constant if its fitted values vary across trials, while a constant that remains the same value across trials is probably stand-alone.

\smallskip
\noindent\textbf{Outcome of a Single Trial.} 
The outcomes of the $k$-th trial are two-fold: (1) the values of the open constants which best fit the given batch of data. We denote these values as vector $\mathbf{c}_k$. (2) the fitness score measuring the goodness-of-fit, denoted as $o_k$.
One typical fitness score is the Negative normalized root mean squared error (NRMSE).
In the example in Fig.~\ref{fig:reduced}, if we fit the reduced expression in (b) to data in trial $T_1$, the best-fitted values are $\mathbf{c}_1=(C_1= 0.1, C_2= 0.35)$. For trial $T_2$, the best-fitted values are  $\mathbf{c}_2=(C_1=0.8, C_2=0.06)$. 
In both trials, the fitness scores (\textit{i.e.}, the NRMSE value) are $0$, indicating no errors.

\smallskip
\noindent\textbf{Outcome of Multiple trials.} 
We let the values of control variables vary across different trials.
This corresponds to changing experimental conditions in real science experiments.
The outcomes of an experiment with $K$ trials are:  (1) $\phi.\mathbf{o}=(o_1, \ldots, o_K)$, where each $o_k$ is the fitness score of trial $k$ and (2) $\phi.\mathbf{c}=(\mathbf{c}_1, \ldots, \mathbf{c}_K)$, the best-fitted values to open constants across trials.

Critical information is obtained by examining the outcomes of a multi-trial control variable experiment. 
First, consistent close-to-zero fitness scores $o_1, \ldots, o_K$ suggest the fitted expression is close to the ground-truth equation in the reduced form. 
Second, given the equation is close to the ground truth, an open constant having similar best-fitted values across $K$ trials $\mathbf{c}_1, \ldots, \mathbf{c}_K$ suggests the constant is stand-alone. Otherwise, it is probably a summary constant.

\subsection{Control Variable Genetic Programming}

The high-level idea of the \method~algorithm is to build more complex symbolic expressions involving more and more variables based on control variable experiments with fewer and fewer controlled variables.

\begin{algorithm}[!t]
   \caption{Control Variable Genetic Programming (CVGP)}\label{alg:cvgp}  
   \begin{algorithmic}[1]
   \Require{GP Pool size $M$; \#generations \texttt{\#Gen}; \#trials $K$; \#expressions in hall-of-fame set \texttt{\#Hof}; mutate probability $P_{mu}$; mate probability $P_{ma}$; operand set $O_p$}.
   \State $\mathbf{v}_c \gets \{x_1, \ldots, x_m\};  \qquad \mathbf{v}_f \gets \emptyset$.
   \State $\mathcal{P}_{gp} \gets \texttt{CreateInitGPPool}(M)$.
   \For{$x_i \in \{x_1, \ldots, x_m\}$ }
        \State $\mathbf{v}_c \gets \mathbf{v}_c \setminus \{x_i\}; \quad \mathbf{v}_f \gets \mathbf{v}_f \cup \{x_i\}$. \Comment{Move $x_i$ from controlled to free variables}
        \State ${\mathcal{D}^o_i} \gets \texttt{DataOracle}(\mathbf{v}_c, \mathbf{v}_f)$.
        \For{$\phi \in \mathcal{P}_{gp} $}
            \State $\{T_k\}_{k=1}^K\gets\texttt{GenData}(\mathcal{D}^o_i)$. \Comment{Query oracle for the trial data}
            
            \State $\phi.\mathbf{o}, \phi.\mathbf{c} \gets\cvexpr(\phi, \mathbf{v}_c, \mathbf{v}_f, \{T_k\}_{k=1}^K)$. \Comment{{Control variable experiments}}
        \EndFor 
    
        \State $\mathcal{P}_{gp}, \mathcal{H} \gets \texttt{GP}(\mathcal{P}_{gp}, \mathcal{D}_i^o, K, M, \texttt{\#Gen}, \texttt{\#Hof}, P_{mu}, P_{ma}, O_p \cup \{\mbox{const}, x_i\})$.

        \For{$\phi \in \mathcal{P}_{gp}$}
            \State \texttt{FreezeEquation}($\phi, \phi.\mathbf{o}$, $\phi.\mathbf{c}$).
        \EndFor        
   \EndFor
    \Return The set of hall-of-fame equations  $\mathcal{H}$.
\end{algorithmic}
\end{algorithm}

To fit an expression of $m$ variables, initially, we control the values of all $m-1$ variables and allow only one variable to vary.
Using Genetic Programming (GP), we find a pool of expressions $\{\phi_{1,1}, \ldots, \phi_{1,M}\}$ which best fit the data from this controlled experiment. 
Notice 
$\phi_{1,1}, \ldots, \phi_{1,M}$ are restricted 
to contain the only one free variable. This fact 
renders fitting them a lot easier than fitting the expressions involving all $m$ variables. 
Next, for each $\phi_{1,l}$, we examine:
\begin{enumerate}[align=left, leftmargin=0pt, labelwidth=0pt, itemindent=!]
    \item if the errors of the fitting are consistently small across all trials. A small error 
implies $\phi_{1,l}$ is close to the ground-truth formula reduced to the one free variable. 
We hence freeze all operands of $\phi_{1,l}$ in this case. Freezing means GP
 in later steps cannot change these operands. 
 \item In the case of a small fitting error, 
we also inspect  the best-fitted values of each open constant in $\phi_{1,l}$ across different trials. 
The constant probably is a summary constant if its values vary across trials.
In other words, these constants represent sub-expressions involving the controlled variables. 
We thus mark these constants as expandable for later steps. The remaining constants are probably stand-alone. Therefore we also freeze them.
\end{enumerate}

After the first step, CVGP adds a second free variable and starts fitting $\{\phi_{2, 1}, \ldots, \phi_{2,M}\}$ using the data from 
 control variable experiments involving the 
 two free variables. 
 Similar to the previous step, all $\phi_{2,l}$ are restricted to only contain the two 
 free variables. 
 Moreover, they can only be mated or mutated by
 GP from the first generation $\{\phi_{1, 1}, \ldots, \phi_{1,M}\}$. 
 The mutation can only happen on non-frozen nodes. 
 After GP, a similar inspection is conducted for every equation in the GP pool, and corresponding variables and/or operands are frozen. 
 This process continues to involve more and more variables. Eventually, the expressions in the GP pool consider all $m$ variables. 

 The whole procedure of CVGP is shown in Algorithm~\ref{alg:cvgp}. Here, $x_1, \ldots, x_m$ are moved
 from the controlled to free variables in numerical order. 
 We agree other orders may boost its performance even further. However, we leave the exploration of this direction as future work. 
 When a new variable becomes free, the control 
 variable experiment $\cvexpr$ needs to be repeated
 for every equation $\phi$ in the GP pool $\mathcal{P}_{gp}$ (Line 5-9 in Algorithm~\ref{alg:cvgp}). 
 This is because the fitness scores and the fitted variable values will both change when the set of controlled variables is updated. 
 Then function {\texttt{GP}} is called. {\texttt{GP}} is a minimally modified genetic programming algorithm for symbolic regression whose pseudo-code is in Algorithm~\ref{alg:gp}. 
The only differences are that it uses data from control variable experiments and the mutation operation at step $i$ only allows to use all the operands, the constant node, and variable $x_i$ at non-frozen nodes.
 Finally, in Lines 12-14 of Algorithm~\ref{alg:cvgp}, $\texttt{FreezeEquation}$ is called for every equation
 in the GP pool. The high-level idea of freezing is discussed above. 
$\mathcal{H}$ is returned as the set of ``hall of fame'' expressions. 

\begin{figure*}[!t]
    \centering
    \includegraphics[width=1.\textwidth]{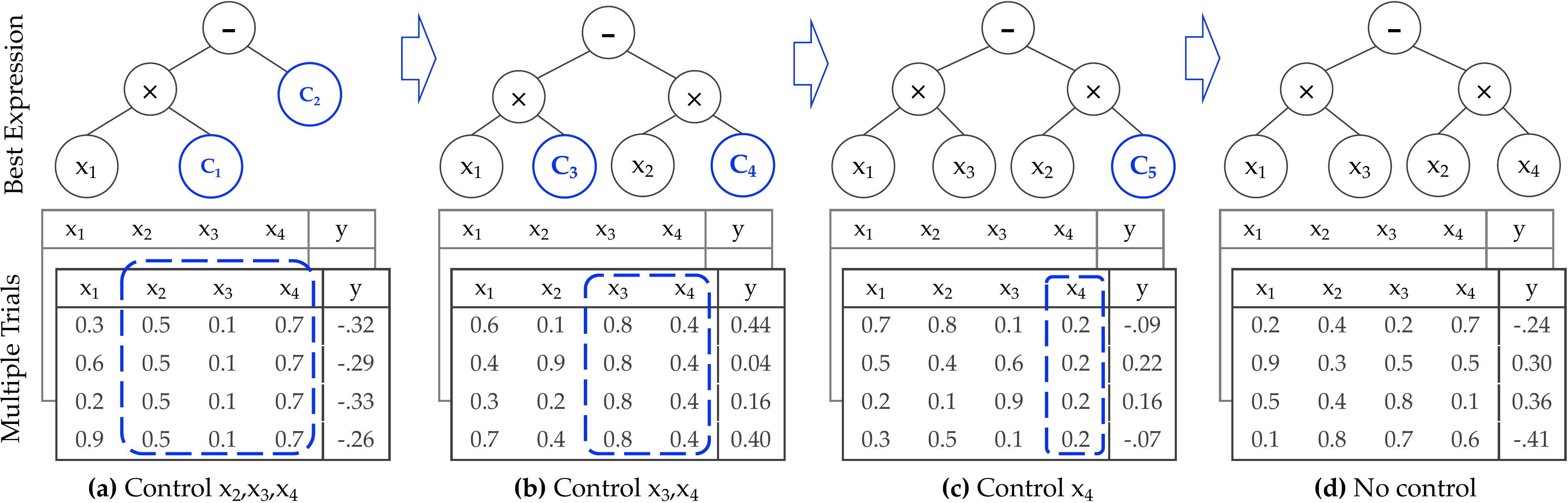}
    \caption{Running example of Algorithm~\ref{alg:cvgp}. 
    \textbf{(a)} Initially, a reduced-form equation $\phi'=C_1 x_1 - C_2$ is found via fitting control variable data in which  $x_2, x_3, x_4$ are held as constants and only  $x_1$ is allowed to vary.  Two leaves nodes $C_1,C_2$ are as summary constants (colored blue).
    \textbf{(b)} This equation is expanded to $C_3 x_1 - C_4 x_2$ in the second stage via fitting the data in which only $x_3, x_4$ are held as constants. 
    \textbf{(c,d)} This process continues until the ground-truth equation $\phi=x_1 x_3 - x_2 x_4$ is found. The data generated for control variable experiment trials in each stage are shown at the bottom. }\label{fig:cvgp}
\end{figure*}

\begin{algorithm}[!t]
   \caption{ {\small\texttt{GP}($\mathcal{P}_{gp}, \mathcal{D}^o, K, M, \texttt{\#Gen}$, $\texttt{\#Hof}$, $P_{mu}$, $P_{ma}$, $O_p$)}}
   \label{alg:gp}
   \begin{algorithmic}[1]
   \Require{Initial GP Pool $\mathcal{P}_{gp}$; Data oracle $\mathcal{D}^o$; \#trials $K$; GP pool size $M$; \#generations \texttt{\#Gen}; \#expressions in hall-of-fame set \texttt{\#Hof}; mutate probability $P_{mu}$; mate probability $P_{ma}$; mutation node library $O_p$}. 
   \State $\mathcal{H} \gets \text{\texttt{TopK}}(\mathcal{P}_{gp}, K=\texttt{\#Hof})$;

   \For{$j \gets 1 \textit{ to } \texttt{\#Gen} $}
   \State $\mathcal{P}_{new} \gets \emptyset$;
   \For{$\phi \in \mathcal{P}_{gp}$}
    \If{with probability $P_{mu}$}\Comment{{Mutation}}
        \State $\phi \gets \texttt{Mutate}(\phi, O_p)$;
        \State $\{T_k\}_{k=1}^K\gets\texttt{GenData}(\mathcal{D}^o)$;
       \State $\phi.\mathbf{o}$, $\phi.\mathbf{c}$ $\gets \cvexpr(\phi, \mathbf{v}_c, \mathbf{v}_f, \{T_k\}_{k=1}^K)$;
    \EndIf
    \State $\mathcal{P}_{new}  \gets \mathcal{P}_{new}  \cup \{\phi\}$;
   \EndFor
   \State $\mathcal{P}_{gp}\leftarrow   \mathcal{P}_{new}$; $\mathcal{P}_{new}\leftarrow \emptyset$;
   \For{$\phi_{l}, \phi_{l+1} \in \mathcal{P}_{gp}$}
    \If{with probability $P_{ma}$} \Comment{{Mating}}
    \State ${\phi_l, \phi_{l+1}} \gets \texttt{Mate}(\phi_l, \phi_{l+1})$;
    \State $\{T_k\}_{k=1}^K\gets\texttt{genData}(\mathcal{D}^o)$;
    \State ${\phi_l}.\mathbf{o}$, ${\phi_l}.\mathbf{c}$  $\gets \cvexpr(\phi_l, \mathbf{v}_c, \mathbf{v}_f, \{T_k\}_{k=1}^K)$. 
    \State ${\phi_{l+1}}.\mathbf{o}$, ${\phi_{l+1}}.\mathbf{c}$  $\gets \cvexpr(\phi_{l+1}, \mathbf{v}_c, \mathbf{v}_f, \{T_k\}_{k=1}^K)$.    
    \EndIf
    \State $\mathcal{P}_{new} \gets \mathcal{P}_{new} \cup \{\phi_l, \phi_{l+1}\}$;
   \EndFor
   
   \State $\mathcal{H}\gets \text{{\texttt{TopK}}}(\mathcal{P}_{new} \cup \mathcal{H}, K=\texttt{\#Hof})$; \Comment{Update the hall of fame set.}
  \State  {$\mathcal{P}_{gp} \gets \texttt{selection}(\mathcal{P}_{new}, M)$};
   \EndFor
  \Return  {GP pool and hall-of-fame $\mathcal{P}_{gp}, \mathcal{H}$.}
\end{algorithmic}
\end{algorithm}

Figure~\ref{fig:cvgp} shows the high-level idea of fitting an equation using CVGP. 
Here the process has four stages, each stage with a decreased number of controlled variables. The trial data in each stage is shown at the bottom and the best expression found is shown at the top. The expandable constants are bold and blue. The readers can see how the fitted equations grow into the final ground-truth equation, with one variable added at a time.

\smallskip
\noindent\textbf{The Availability of a Data Oracle.}
A crucial assumption behind the success of \method~is the availability of a {data oracle $\mathcal{D}^o$} that returns a (noisy) observation of the dependent variable with input variables in $\mathbf{v}_c$ controlled and $\mathbf{v}_f$ free. 

This differs from the classical setting of symbolic regression, where a dataset is obtained before learning \cite{matsubara2022rethinking,doi:10.1080/15598608.2007.10411855}.
Such a data oracle represents conducting control variable experiments in the real world, which can be expensive. 

However, we argue that the integration of experiment design in the discovery of scientific knowledge is indeed the main driver of the successes of \method. 
This idea has received tremendous success in early works~\cite{DBLP:conf/ijcai/Langley77,DBLP:conf/ijcai/Langley79,DBLP:conf/ijcai/LangleyBS81} but unfortunately has been largely forgotten in today's symbolic regression community. 
Our work does not intend to show the superiority of one approach. Instead, we would like to point out that carefully designed experiments can improve any method, and GP is used as an example. 
We acknowledge that fully controlled experiments may be difficult in some scenarios.
In cases where it is difficult to obtain
such a data oracle, we propose to leverage counterfactual reasoning to construct datasets corresponding to control variable trials by sampling from existing training data.
We leave such effort as future work.

\subsection{Theoretical Analysis}
We demonstrate in this section that the idea of control variable experiments may bring an exponential reduction in the search space for particular classes of symbolic expressions. 
To see this, we assume that the learning algorithm follows a search order from simple to complex symbolic expressions. 
\begin{definition}
\textit{The search space of symbolic expression trees of $l$ nodes} $S(l)$ is made up of all symbolic expression trees involving at most $l$ nodes.
\end{definition}

\begin{lemma}\label{lem:searchspace}
Assuming that all operands are binary, $o$ is the number of operands, and $m$ is the number of input variables. The size of the search space of symbolic expression trees of $l$ nodes scales exponentially; more precisely at $\mathcal{O}((4(m+1)o)^{\frac{l-1}{2}})$ and $\Omega((4(m+1)o)^{\frac{l-1}{4}})$. 
\end{lemma}
\begin{proof}
    Because all operands are binary, a symbolic expression tree of $l$ nodes has $\frac{l+1}{2}$ leaves and $\frac{l-1}{2}$ internal nodes. 
    The number of binary trees of $\frac{l-1}{2}$ internal nodes is given by the Catalan number $C_{{(l-1)}/{2}} = \frac{l-1}{{l+1}}{l-1 \choose \frac{(l-1)}{2}}$, which asymptotically scales at $\frac{2^{{l-1}}}{{(\frac{l-1}{2})}^{3/2}\sqrt{\pi}}$. 
    A symbolic expression replaces each internal node of a binary tree with an operand and replaces each leaf with either a constant or one of the input variables. 
    Because there are $o$ operands and $m$ input variables, the total number of different symbolic expression trees involving $l$ nodes is given by:
    \begin{align}
    A(l) = C_{(l-1)/{2}} (m+1)^{\frac{l+1}{2}} o^\frac{l-1}{2}
    &\sim \frac{(4(m+1)o)^{\frac{l-1}{2}}}{ \left(\frac{l-1}{2}\right)^{3/2}}.
    \end{align}
    Hence, the total number of trees up to $l$ nodes is:
    \begin{align}
    S(l) = \sum_{i=0}^{(l-1)/2} A(2i+1) \sim \sum_{i=0}^{(l-1)/2} \frac{(4(m+1)o)^{i}}{ i^{3/2}}.
    \end{align}
    When $i$ is sufficiently large,
    \begin{align}
    (4(m+1)o)^{i/2} \leq \frac{(4(m+1)o)^{i}}{ i^{3/2}} \leq (4(m+1)o)^{i}.
    \end{align}
    Therefore, $S(l) \leq \sum_{i=0}^{(l-1)/2} (4(m+1)o)^{i} \in \mathcal{O}((4(m+1)o)^{(l-1)/2})$, and
    $S(l) \geq (4(m+1)o)^{(l-1)/2} / ((l-1)/2)^{3/2} \geq (4(m+1)o)^{(l-1)/4}$, which implies 
    $S(l) \in \Omega((4(m+1)o)^{\frac{l-1}{4}})$.
\end{proof}
The proof of Lemma~\ref{lem:searchspace} mainly involves counting binary trees.
For our purposes, it is sufficient to know that the size is exponential in $l$.

\begin{definition}[Simple to complex search order]
A symbolic regression algorithm follows a simple to complex search order if it expands its search space from short to long symbolic expressions; \textit{i.e.}, first search for the best symbolic expressions in $S(1)$, then in $S(2) \setminus S(1)$, etc. 
\end{definition}

It is difficult to quantify the search order of any symbolic regression algorithms. 
However, we believe the simple to complex order reflects the search procedures of a large class of symbolic regression algorithms, including our \method. 
In fact, \cite{DBLP:journals/tcyb/ChenXZ22a} explicitly use regularizers to promote the search of 
simple and short expressions. 
Our \method~follows the simple to complex search order approximately.
Indeed, it is possible that genetic programming encounters more complex equations
before their simpler counterparts. 
However, in general, the expressions are built from simple to complex equations by mating and mutating operations in genetic programming algorithms. 

\begin{proposition}[Exponential Reduction in the Search Space] There exists a symbolic expression $\phi$ 
of $(4m-1)$ nodes, a normal symbolic regression
algorithm following the simple to complex search order has to explore a search space whose size is exponential in $m$ to find the expression, while 
\method~following the simple to complex order only expands $\mathcal{O}(m)$ constant-sized search spaces.
 
\begin{proof}
Consider a dataset generated by the ground-truth symbolic expression made up of 2 operands ($+, \times$), $2m$ input variables and $(4m-1)$ nodes:
\begin{align}
(x_1 + x_2) (x_3 + x_4) \ldots (x_{2m-1} + x_{2m}).
\end{align}
To search for this symbolic regression, a normal algorithm following the simple to complex order needs to consider all 
expressions up to $(4m-1)$ nodes. According to Lemma~\ref{lem:searchspace}, the normal algorithm has a search space of at least $\Omega((16m + 8)^{m-1/2})$, which is exponential in $m$.

On the other hand, {in the first step of \method}, $x_{2}, \ldots, x_{2m}$ are controlled and only $x_1$ {is free}.
In this case, the ground-truth equation in the reduced form is
\begin{align}
(x_1 + C_1) D_1,
\end{align}
in which both $C_1$ and $D_1$ are summary constants. Here $C_1$ represents $x_2$ and $D_1$ represents $(x_3 + x_4) \ldots (x_{2m-1} + x_{2m})$ in the control variable experiments. 
The reduced equation is quite simple under
the controlled environment. \method~should be able to find 
the ground-truth expression exploring search space $S(5)$. 

Proving using induction. In step $2i~(1\leq i \leq m)$, variables $x_{2i+1}, x_{2i+2}, \ldots, x_{2m}$ are held as constants, and $x_1, \ldots, x_{2i}$ are allowed to vary. 
The ground-truth expression in the reduced form found in the previous $(2i-1)$-th step is:
\begin{align}
(x_1 + x_2) \ldots (x_{2i-1} + C_{2i-1}) D_{2i-1}.
\label{eq:true2i-1}
\end{align}
\method~needs to extend this equation to be the ground-truth expression in the reduced form for the $2i$-th step, which is:
\begin{align}
(x_1 + x_2) \ldots (x_{2i-1} + x_{2i}) D_{2i}.
\label{eq:true2i}
\end{align}
We can see the change is to replace the summary constant $C_{2i-1}$ to $x_{2i}$.
Assume {the data is noiseless} and \method~can confirm expression \eqref{eq:true2i-1} is the ground-truth reduced-form expression for the previous step. This means all the operands and variables will be frozen by \method, and only $C_{2i-1}$ and $D_{2i-1}$ are allowed to be replaced by new expressions. 
Assume \method follows the simple to complex search order, it should find the ground-truth expression \eqref{eq:true2i} by searching replacement expressions of lengths up to 1.

Similarly, in step $2i+1$, assume \method confirms the ground-truth expression 
in the reduced form in step $2i$, \method also only needs to search in constant-sized spaces to
find the new ground-truth expression. Overall, we can see only $\mathcal{O}(m)$ searches in constant-sized spaces are required for \method to find the final ground-truth expression. 
\end{proof}
\end{proposition}

\section{Related Work}
\noindent\textbf{Symbolic Regression.} 
Symbolic Regression is proven to be NP-hard~\cite{journal/tmlr/virgolin2022}, due to the search space of all possible symbolic expressions being exponential in the number of input variables.  
Early works in this domain are based on heuristic search~\cite{LANGLEY1981DataDiscovery,LENAT1977ubiquity}.
Generic programming turns out to be effective in searching for good candidates of symbolic expressions~\cite{journal/2020/aifrynman,DBLP:conf/gecco/VirgolinAB19,DBLP:journals/corr/abs-2203-08808,DBLP:conf/gecco/HeLYLW22}. Reinforcement learning-based methods propose a risk-seeking policy gradient to find the expressions~\cite{DBLP:conf/iclr/PetersenLMSKK21,DBLP:journals/corr/abs-2205-11107,DBLP:conf/nips/MundhenkLGSFP21}. Other works reduced the combinatorial search space by considering the composition of base functions, \textit{e.g.} Fast function extraction~\cite{mcconaghy2011ffx} and elite bases regression~\cite{DBLP:conf/icnc/ChenLJ17}.  
In terms of the families of expressions, research efforts  have been devoted to searching for polynomials with single or two variables~\cite{DBLP:journals/gpem/UyHOML11}, time series equations~\cite{DBLP:conf/icml/BalcanDSV18}, and also equations in physics~\cite{journal/2020/aifrynman}. 
Multi-variable symbolic regression is more challenging because the search space increases exponentially with respect to  the number of independent variables. Our \method is a tailored algorithm to solve multi-variable symbolic regression problems. 


\smallskip
\noindent\textbf{AI-driven Scientific Discovery.} Recently AI has been highlighted to enable scientific discoveries in diverse domains~\cite{langey1988scientificdiscovery,Bengio2022Nature}. 
Early work in this domain focuses on learning logic (symbolic) representations~\cite{BRADLEY2001reasoning,Bridewell2008inductive}.
Recently, learning Partial Differential Equations (PDEs) from data has also been studied extensively 
~\cite{Dzeroski1995lagrange,brunton2016sparse,PhysRevE.100.033311,doi:10.1098/rspa.2018.0305,iten2020discovering,DBLP:conf/nips/CranmerSBXCSH20,Raissi20Fluid,RAISSI2019PhysicsInformedNN,Liu21AIPoincare,nanovoid_tracking,chen2018neural}. 
In this domain, a line of works develops robots that automatically refine the hypothesis space, some with human interactions~\cite{Valdes1994,king2004functional,king2009autosci}.
These works are quite related to ours because they also actively probe the hypothesis spaces, albeit they are in  biology and chemistry.

\smallskip
\noindent\textbf{Active Learning and Reasoning.}
Active learning considers querying data points actively to maximize the learning performance~\cite{DBLP:journals/ftml/Hanneke14,golovin2010near}. 
Our approach is related to active learning because control variable experiments can be viewed as a way to actively collect data. 
However, besides active data collection, our \method~builds
simple to complex models, which is not in active learning. 

\smallskip
\noindent\textbf{Meta-reasoning -- Thinking Fast and Slow}. 
The co-existence of fast and slow cognition systems marks an interesting side of human intelligence~\cite{kahneman2011thinking,DBLP:conf/nips/AnthonyTB17,DBLP:conf/aaai/BoochFHKLLLMMRS21}. 
Our \method is motivated by this dual cognition process. 
In essence, we argue instead of entirely relying on the brute-force way of learning using big data and heavy computation (fast thinking), incrementally expanding from reduced-form equations to the full equation may result in better outcomes (slow thinking).

\smallskip
\noindent\textbf{Causality.}
Control variable experiments are closely related to the idea of intervention, which is commonly used to discover causal relationships \cite{simon1954spurious,langley2019scientific,glymour2014discovering,jaber2022causal,pearl2009causality}. 
However, we mainly use control variable experiments to accelerate symbolic regression, which still identifies correlations.

\section{Experiments}
In this section, we demonstrate \method finds the symbolic expressions with the smallest Normalized Mean-Square Errors (NMSEs) 
among all 7 competing approaches on 21 noiseless benchmark datasets (in Table~\ref{tab:summary-nmse}) and 
20 noisy benchmark datasets (in Table~\ref{tab:summary-noisy}). 
In the ablation studies, we show our \method is consistently better than the baselines when evaluated in different evaluation metrics, evaluating different quantiles of the NMSE metric, with different amount of Gaussian noise added to the data (Figure \ref{fig:evalucate-metric}, more complete results in Figure \ref{fig:Quartile-nmse-noiseless-full} and \ref{fig:Quartile-nmse-noisy-full} in the appendix). 
For simple datasets, our approach can recover the ground-truth symbolic expressions. 
Table~\ref{tab:recovery} shows our \method has a higher rate of recovering the ground-truth expressions than baselines.

\begin{table}[!t]
    \centering
    \scalebox{0.75}{
    \begin{tabular}{c|rr|rr|rr|rr|rr|rr|rr}
    \toprule
        Dataset &\multicolumn{2}{c|}{\method (ours)} & \multicolumn{2}{c|}{GP}&\multicolumn{2}{c|}{DSR} &\multicolumn{2}{c|}{PQT} &\multicolumn{2}{c|}{VPG} & \multicolumn{2}{c|}{GPMeld}& \multicolumn{2}{c}{Eureqa}  \\ 
       configs & $50\%$&  $75\%$& $50\%$&  $75\%$& $50\%$&  $75\%$& $50\%$&  $75\%$& $50\%$&  $75\%$& $50\%$&  $75\%$& $50\%$&  $75\%$\\\midrule
     (2,1,1) & <$\mathbf{1e\text{-}6}$ & <$\mathbf{1e\text{-}6}$ & $2.19e{\text{-}3}$ & $7.91e{\text{-}2}$ & $3.507$ & $4.787$ & $0.262$ & $1.16$ & $0.359$ & $19.16$ & $0.273$ & $1.356$& <$\mathbf{1e\text{-}6}$ & <$\mathbf{1e\text{-}6}$ \\
     (3,2,2) & $0.001$ & $0.004$ & $0.015$ & $0.135$ & $1.53$ & $43.09$ & $0.58$ & $1.13$ & $0.83$ & $1.32$  & $1.06$ & $2.18$ & <$\mathbf{1e\text{-}6}$ & <$\mathbf{1e\text{-}6}$\\
     (4,4,6)& $\mathbf{0.008}$& $0.059$& $0.012$& $\mathbf{0.054}$& $1.006$& $1.249$& $1.006$& $2.459$& $1.221$& $2.322$& $1.127$& $2.286$& $1.191$& $6.001$ \\
     (5,5,5) & $\mathbf{0.011}$ & $\mathbf{0.019}$ & $0.025$ & $0.177$ & $1.038$ & $8.805$ & $1.048$  & $4.736$ & $1.401$ & $38.26$ & $1.008$ & $1.969$ & $0.996$  & $6.340$\\
     (5,5,8) & $\mathbf{0.007}$ & $\mathbf{0.013}$ & $0.010$ & $0.017$ & $1.403$ & $5.161$ & $1.530$ & $41.27$ & $4.133$ & $27.42$ & $1.386$ & $8.092$ & $1.002$ & $1.495$\\
       (6,6,8) & $\mathbf{0.044}$& $\mathbf{0.074}$ &  $0.058$ & $0.200$ &  $1.963$ & $90.53$ & $4.212$ & $8.194$ &  $4.425$ & $22.91$ &$15.58$ & $269.6$  &$1.005$ & $1.150$    \\
     (6,6,10) & $\textbf{0.012}$ & $\textbf{0.027}$& $0.381$ & $0.820$& $1.021$ & $1.036$& $1.006$ & $1.048$& $1.003$ & $1.020$& $1.022$ & $1.689$& $1.764$ & $49.041$ \\
        \hline
         \multicolumn{15}{c}{\textbf{(a)} Datasets containing operands $\{\texttt{inv},+,-,\times\}$} \\
    \hline
     (2,1,1) &  $1.06e{\text{-}4}$ & $6.69e{\text{-}2}$ & $7.56e{\text{-}4}$ & $7.72e{\text{-}2}$ & $1.87$ & $8.16$ & $0.20$ & $0.22$ & $2.38$ & $6.14$ & <$\mathbf{1e\text{-}6}$ & <$\mathbf{1e\text{-}6}$ & <$\mathbf{1e\text{-}6}$ & <$\mathbf{1e\text{-}6}$\\
    (3,2,2) & $0.005$ & $0.123$ & $0.023$ & $0.374$ & $0.087$ & $0.392$ & $0.161$ & $0.469$ & $0.277$  & $0.493$ & $0.112$ & $0.183$ & <$\mathbf{1e\text{-}6}$ & <$\mathbf{1e\text{-}6}$\\
    (4,4,6) & $\textbf{0.028}$ & $0.132$ & $0.044$ & $\textbf{0.106}$ & $2.815$ & $9.958$ & $2.381$ & $13.844$ & $2.990$ & $11.316$ & $1.670$ & $2.697$ & $0.024$ & $0.122$ \\
   (5,5,5)& ${0.086}$ & ${0.402}$ & $\mathbf{0.063}$ & $\mathbf{0.232}$ & $2.558$ & $ 3.313$ & $2.168$ & $ 2.679$ & $1.903$ & $ 2.780$ & $1.501$ & $2.295$ & $0.158$ & $0.377$\\
 (5,5,8) & $\mathbf{0.014}$ & $\mathbf{0.066}$ & $0.102$ & $ 0.683$ & $ 2.535$ & $2.933$ & $ 2.482$ & $2.773$ & $2.440$ & $3.062$ & $2.422$ & $3.853$ & $0.284$ & $0.514$\\
 (6,6,8) &$\mathbf{0.066}$ & $\mathbf{0.166}$ & $0.127$ & $ 0.591$ & $ 0.936$ & $1.079$ & $ 0.983$ & $1.053$ & $ 0.900$ & $1.018$ & $0.964$ & $1.428$ & $0.433$ & $1.564$\\
(6,6,10) &  $\textbf{0.104}$&  $\textbf{0.177}$&  $0.159$&  $0.230$&  $6.121$&  $16.32$&  $5.750$&  $16.29$&  $3.857$&  $19.82$&  $7.393$&  $21.709$&  $0.910$&  $1.927$ \\
 \hline
  \multicolumn{15}{c}{\textbf{(b)} Datasets containing operands $\{\sin, \cos,+,-,\times\}$.}  \\
\hline
   (2,1,1) &  <$\mathbf{1e\text{-}6}$ &  $0.004$ & <$\mathbf{1e\text{-}6}$ & $0.76$ & $0.032$ & $4.778$ & $0.038$  & $4.782$ & $0.115$ & $4.095$ & $0.008$ & $5.859$ & <$\mathbf{1e\text{-}6}$ & <$\mathbf{1e\text{-}6}$\\
    (3,2,2) & $0.039$ & $0.083$   & $0.043$ & $0.551$  & $0.227$ & $7.856$ & $0.855$ & $2.885$ & $0.233$ & $0.400$ & $0.944$  & $1.263$ & <$\mathbf{1e\text{-}6}$ & <$\mathbf{1e\text{-}6}$\\
 (4,4,6)& $\mathbf{0.015}$ & $\mathbf{0.121}$ & $0.042$ & $ 0.347$ & $1.040$ & $ 1.155$ & $1.039$ & $ 1.055$ & $1.049$ & $ 1.068$  & $1.886$ & $4.104$ & $0.984$ & $1.196$\\

 (5,5,5)& $\mathbf{0.038}$ & $\mathbf{0.097}$ & $0.197$ & $ 0.514$ & $3.892$ & $ 69.98$ & $4.311$ & $ 23.66$ & $5.542$ & $ 8.839$ & $9.553$ & $16.92$ & $0.901$& $1.007$\\
 (5,5,8)& $\mathbf{0.050}$ & $\mathbf{0.102}$ & $0.111$ & $ 0.177$ & $2.379$ & $ 2.526$ & $1.205$ & $ 2.336$ & $1.824$ & $ 2.481$ & $1.142$ & $1.874$ & $1.002$ & $2.445$ \\

 (6,6,8) & $\mathbf{0.029}$ & $\mathbf{0.038}$ & $0.091$ & $ 0.151$ & $1.605$ & $ 8.005$  & $1.718$ & $7.783$ & $4.691$ & $39.03$  & $1.398$ & $16.60$ & $1.001$ & $1.008$\\
 (6,6,10) & $\mathbf{0.018}$ & $\mathbf{0.113}$ & $0.087$ & $0.194$ & $2.083$ & $23.57$ & $1.797$ & $4.521$ & $1.888$ & $35.45$ & $2.590$ & $8.784$ & $1.001$ & $1.008$ \\
 \hline
  \multicolumn{15}{c}{\textbf{(c)} Datasets containing operands $\{\sin, \cos,\texttt{inv},+,-,\times\}$.} 
    \end{tabular}}
    \caption{Median (50\%) and 75\%-quantile NMSEs of the symbolic expressions found by all the algorithms on \textit{noiseless} benchmark datasets. 
    Our \method finds symbolic expressions with the smallest NMSEs. }
    \label{tab:summary-nmse}
\end{table}

\subsection{Experimental Settings} \label{sec:exp-set}
\noindent\textbf{Datasets.} 
To highlight the performance of \method in regressing multi-variable expressions, we consider synthesized datasets,  involving randomly generated expressions with multiple variables. 
A dataset is labeled by the ground-truth equation that generates it. 
The ground-truth equations we consider are multi-variable polynomials characterized by their operands and a tuple $(a, b, c)$. Here $a$ is the number of independent variables. $b$ is the number of singular terms. A singular term can be an independent variable, such as $x_1$, or a unary operand on a variable, such as $\sin(x_1)$. $c$ is the number of cross terms. They look like $C_1 x_3 x_4$ or $C_2 \sin(x_1) \texttt{inv}(x_5)$, etc. Here $C_1, C_2$ are randomly generated constants. 
The tuples and operands listed in different tables and charts indicate how the ground-truth expressions are generated. 
%
%
For each dataset configuration, we repeat our experiments 10 times, each time with a randomly generated symbolic expression of the given configuration. 
For noiseless datasets, the output is exactly the evaluation of the ground-truth expression. For noisy datasets, the output is further perturbed by Gaussian noise of zero means and a given standard deviation. 

\begin{table}[!t]
    \centering
    \scalebox{0.83}{
    \begin{tabular}{c|rr|rr|rr|rr|rr|rr}
    \hline
      Dataset &\multicolumn{2}{c|}{\method (ours)} & \multicolumn{2}{c|}{GP}&\multicolumn{2}{c|}{DSR} &\multicolumn{2}{c|}{PQT} &\multicolumn{2}{c|}{VPG} & \multicolumn{2}{c}{GPMeld} \\ 
       configs & $50\%$&  $75\%$& $50\%$&  $75\%$& $50\%$&  $75\%$& $50\%$&  $75\%$& $50\%$&  $75\%$& $50\%$&  $75\%$\\\hline
   (2,1,1) & ${0.198}$ & ${0.490}$ & $\mathbf{0.024}$ & $\mathbf{0.053}$ & $0.032$ & $3.048$ & $0.029$ & $0.953$ & $0.041$ & $0.678$ & $0.387$ & $22.806$\\
   (4,4,6)  & $\mathbf{0.036}$ & $\mathbf{0.088}$   & $0.038$ & $0.108$ & $1.163$ & $3.714$ & $1.016$ & $1.122$ & $1.087$ & $1.275$ & $1.058$ & $1.374$\\
     (5,5,5)  & $0.076$ & $0.126$ & $\mathbf{0.075}$ &$\mathbf{0.102}$  & $1.028$ & $2.270$  & $1.983$ & $4.637$  & $1.075$ & $2.811$ & $1.479$ & $2.855$ \\ 
    (5,5,8) &  $\mathbf{0.061}$ & $\mathbf{0.118}$ & $0.121$ & $0.186$  & $1.004$ & $1.013$ & $1.005$ &$1.006$ & $1.002$ & $1.009$ & $1.108$ & $2.399$ \\
    (6,6,8)  & $\mathbf{0.098}$ & $\mathbf{0.144}$ & $0.104$ & $0.167$ & $1.006$ & $1.027$ & $1.006$ & $1.020$ & $1.009$ & $1.066$ & $1.035$ & $2.671$  \\
    (6,6,10)  & $\mathbf{0.055}$ & $\mathbf{0.097}$ & $0.074$ & $0.132$ & $1.003$  & $1.009$  & $1.005$  & $1.008$  & $1.004$  & $1.015$  & $1.021$  & $1.126$ \\
    \hline
    \multicolumn{13}{c}{\textbf{(a)} Datasets containing operands $\{\sin, \cos,\texttt{inv},+,-,\times\}$.} \\
    \hline
     (2,1,1) & $\textbf{0.049}$ & $0.812$ & $0.103$ & $0.263$ &  $0.069$ & $0.144$ & $0.066$ & $\mathbf{0.093}$ & $0.094$ & $0.416$ & $0.066$ & $0.118$ \\
      (3,2,2) & $\mathbf{0.098}$ & $\mathbf{0.165}$ & $0.108$ & $0.425$ &  $0.350$ & $0.713$ & $0.351$ & $1.831$ &  $0.439$ & $0.581$ & $0.102$ & $0.597$\\
     (4,4,6)  & $\mathbf{0.078}$ & $\mathbf{0.121}$ & $0.120$ & $0.305$& $7.056$ & $16.321$ & $5.093$ & $19.429$ & $2.458$ & $13.762$ & $2.225$ & $3.754$ \\
    (5,5,5)   & $\mathbf{0.067}$ & $\mathbf{0.230}$ & $0.091$ & $0.313$& $32.45$ & $234.31$ & $36.797$ & $229.529$ & $14.435$ & $46.191$ & $28.440$ & $421.63$ \\ 
    (5,5,8)   & $\mathbf{0.113}$  & $\mathbf{0.207}$  & $0.119$  &$0.388$ & $195.22$ & $573.33$ & $449.83$ & $565.69$ & $206.06$ & $629.41$ & $363.79$ & $666.57$  \\
    (6,6,8)   & $\mathbf{0.170}$  & $\mathbf{0.481}$  & $0.186$  &  $0.727$  & $1.752$ & $3.824$ & $4.887$ & $15.248$ & $2.396$ & $7.051$ & $1.478$ & $6.271$ \\
    (6,6,10)   &  $\mathbf{0.161}$  & $\mathbf{0.251}$  & $0.312$  &$0.342$  & $11.678$ & $26.941$ & $5.667$ & $24.042$ & $7.398$ & $25.156$ & $11.513$ & $28.439$  \\
   \hline
    \multicolumn{13}{c}{\textbf{(b)} Datasets containing operands $\{\sin, \cos,+,-,\times\}$.} \\
\hline
     (2,1,1) & $0.241$ & $0.873$ & $0.102$ & $1.0018$ & $0.440$ & $1.648$ & $0.757$ & $9.401$ & $0.2142$ & $3.349$ & $\mathbf{0.0002}$ & $\mathbf{0.0007}$ \\ 
     (3,2,2) & $0.049$ & $\mathbf{0.113}$ & $\mathbf{0.023}$ & $0.166$ & $0.663$ & $2.773$ & $1.002$ & $1.992$ & $0.969$ & $1.310$ & $0.413$ & $2.510$ \\ 
    (4,4,6)   & $\mathbf{0.141}$ & $\mathbf{0.220}$ & $0.238$ & $0.662$  & $1.031$ & $1.051$ & $1.297$ & $1.463$ & $1.051$ & $1.774$& $1.093$ & $1.769$ \\ 
   (5,5,5)   & $\mathbf{0.157}$ & ${0.438}$   & $0.195$ & $\mathbf{0.337}$ & $1.098$ & $3.617$ & $1.018$ & $5.296$ & $1.012$ & $1.27$  & $1.036$  & $3.617$\\
   (5,5,8)   & $\mathbf{0.122}$ & $\mathbf{0.153}$ & $0.166$ & $0.186$  & $1.009$ & $1.103$ & $1.017$ & $1.429$ & $1.007$ & $1.132$ & $1.07$ & $2.904$\\
   (6,6,8)   & $\mathbf{0.209}$ & $\mathbf{0.590}$ & $\mathbf{0.209}$ & $0.646$  & $1.003$ & $1.153$ & $1.047$ & $1.134$ & $1.059$ & $1.302$ & $1.029$ & $3.365$\\
     (6,6,10)   & ${0.139}$ & ${0.232}$ & $\mathbf{0.073}$ & $\mathbf{0.159}$  & ${1.654}$ & $3.408$ & $1.027$ & $1.069$ & $1.009$ & $1.654$ & $1.445$ & $2.106$\\
 \hline
  \multicolumn{13}{c}{\textbf{(c)} Datasets containing operands $\{\sin, \cos,\texttt{inv},+,-,\times\}$.} \\
    \end{tabular}}
    \caption{Median (50\%) and 75\%-quantile NMSE values of the symbolic expressions found by all the algorithms on several \textit{noisy} benchmark datasets (Gaussian noise with zero mean and standard deviation 0.1 is added). 
    Our \method finds symbolic expressions with the smallest NMSEs.}
    \label{tab:summary-noisy}
\end{table}
\begin{figure}[!ht]
    \centering
    \includegraphics[width=0.32\linewidth]{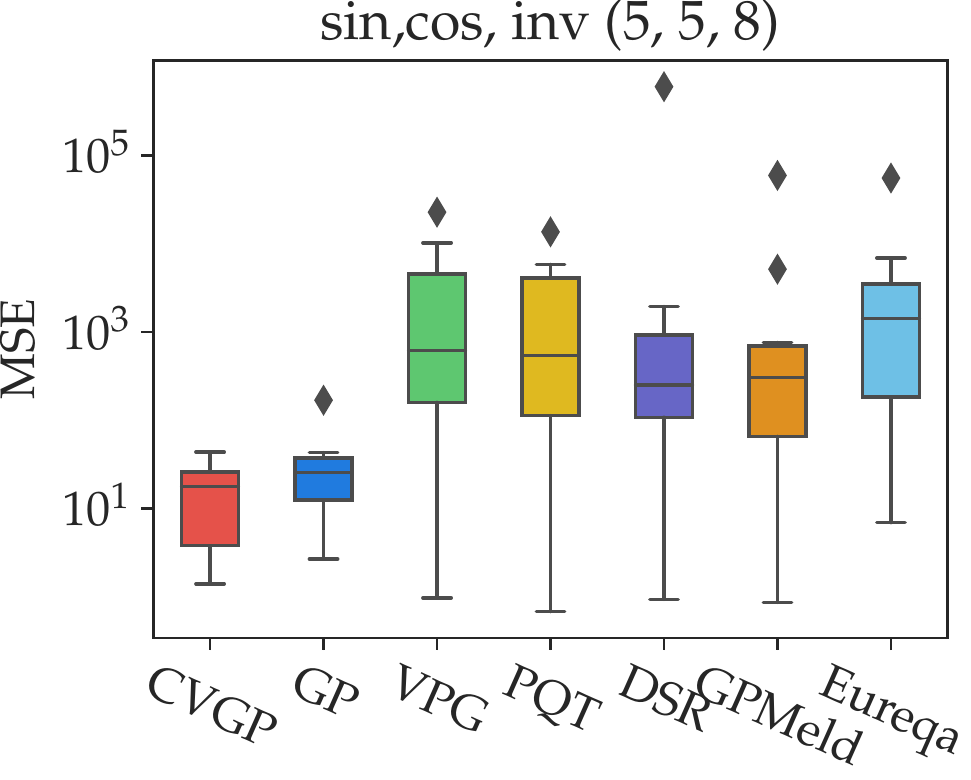}
    \includegraphics[width=0.32\linewidth]{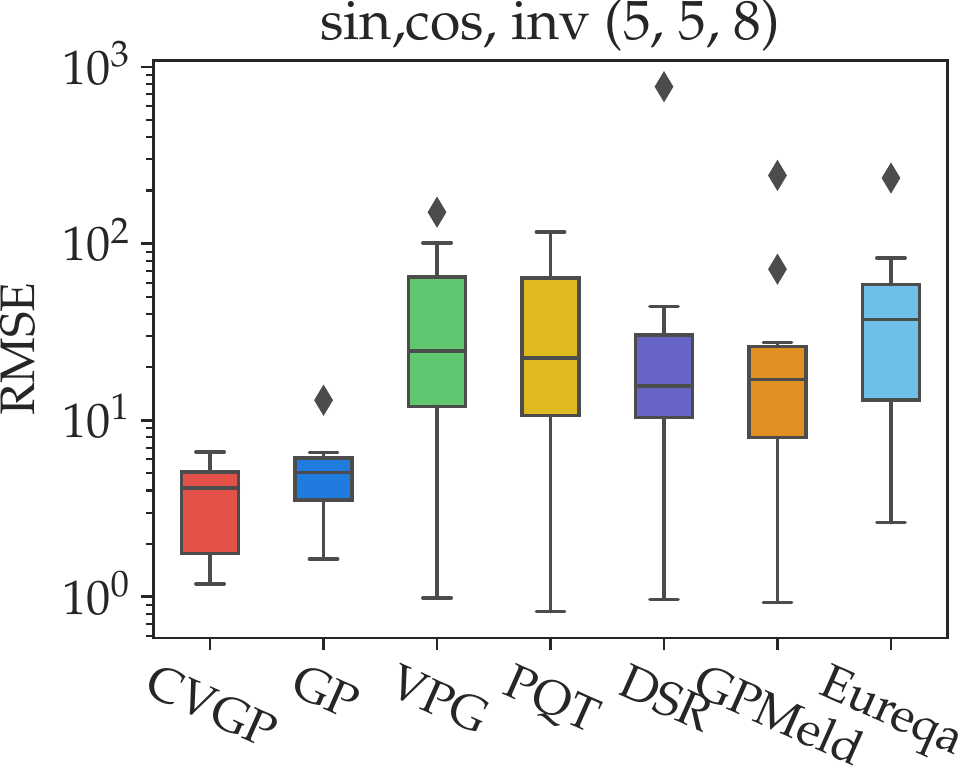}
    \includegraphics[width=0.32\linewidth]{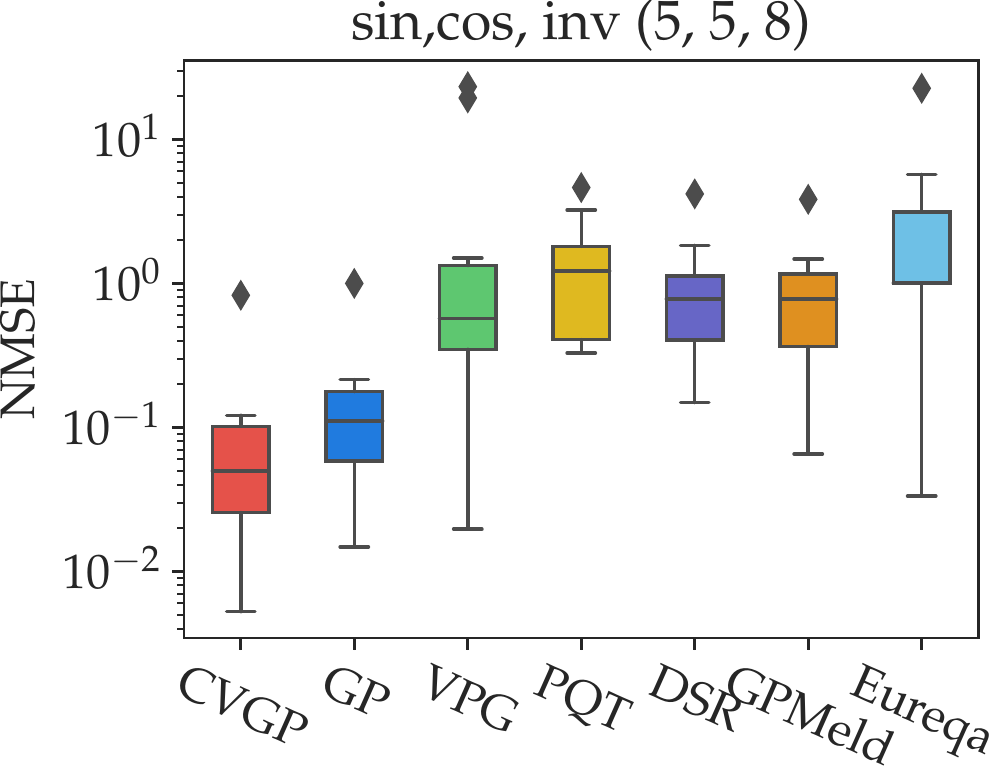} \\
    \hspace{4em} (a)\hspace{10em} (b) \hspace{12em} (c) \\
    \includegraphics[width=0.32\linewidth]{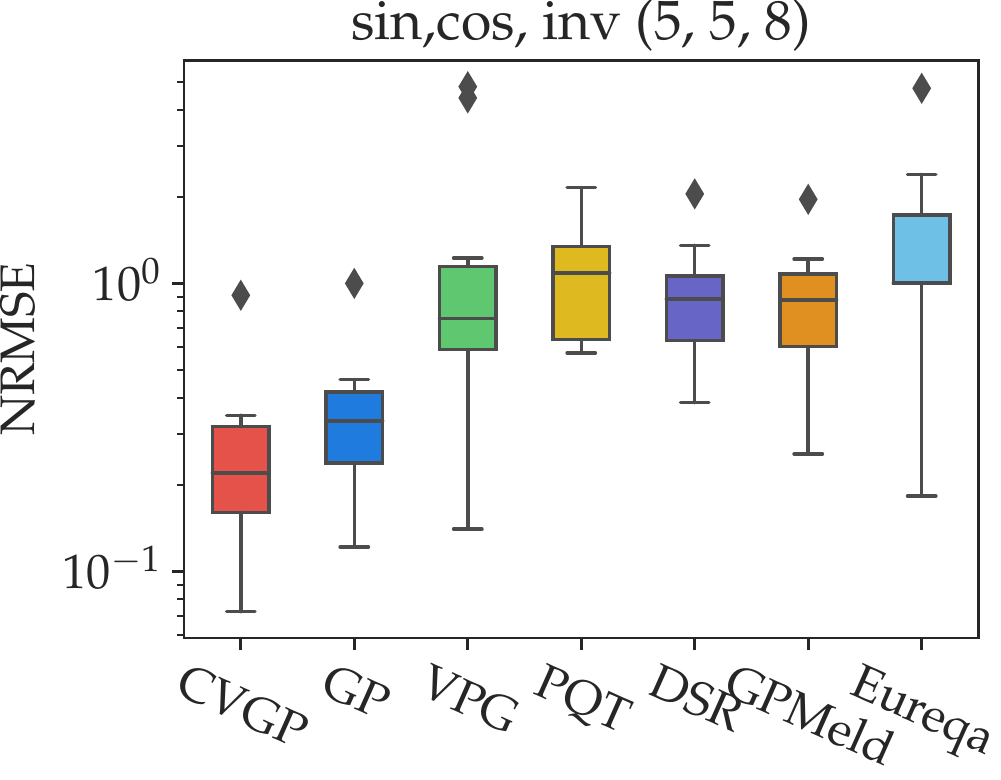}
    \includegraphics[width=0.31\linewidth]{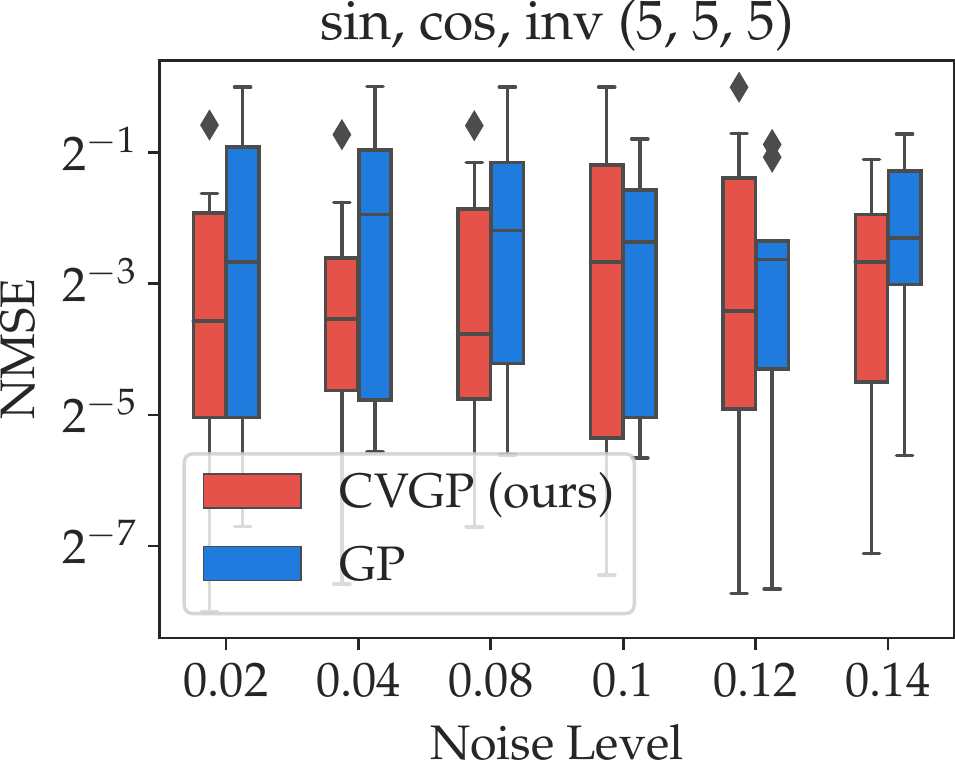}
    \includegraphics[width=0.31\linewidth]{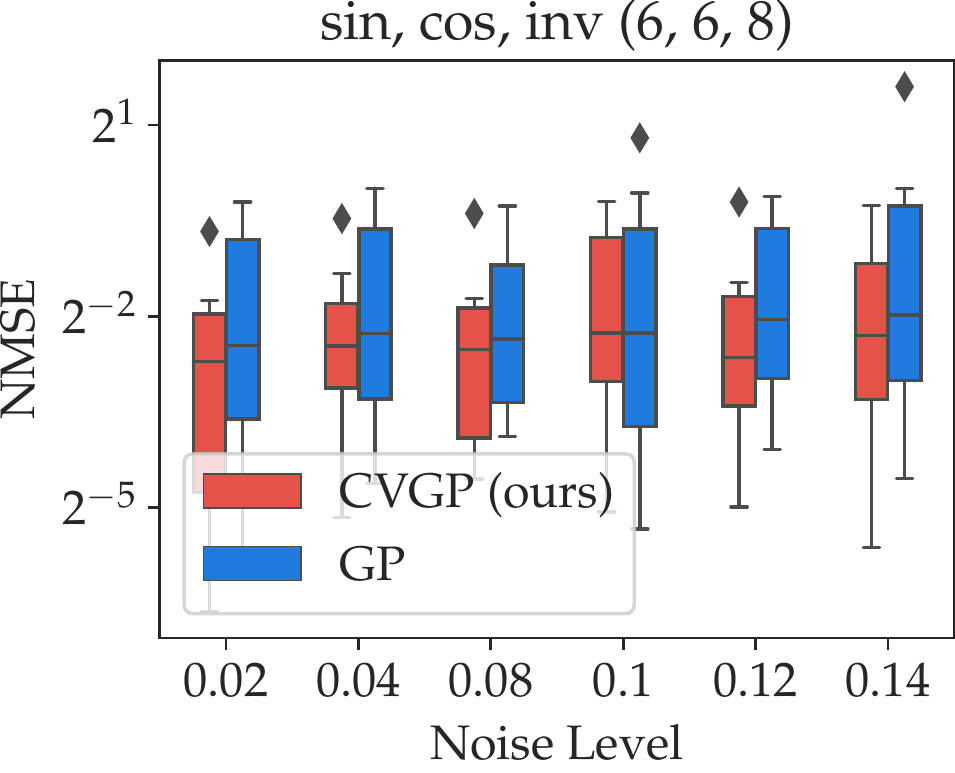} \\
    \hspace{4em} (d)\hspace{10em} (e) \hspace{12em} (f) \\
    \caption{\textbf{(a-d)} Box plots in different evaluation metrics of the expressions found by different algorithms on the noiseless dataset. 
    %
    %
    \textbf{(e-f)} Box plots in NMSE values for the expressions found by \method~and GP over benchmark datasets with different noise levels. 
    Our \method is consistently the best regardless of the evaluation metrics and noise levels.
    }
    \label{fig:evalucate-metric}
\end{figure}

\noindent\textbf{Remarks on Public Available Datasets.} 
Most public datasets are black-box \cite{la2021contemporary}, containing randomly generated input and output pairs of an unknown symbolic equation. 
The point of our paper is to show customized collected control variable experiment data improves symbolic regression, and hence we cannot use these randomly generated data. 
In addition, most datasets are on equations of a small number of independent variables. 
We intentionally test  on benchmark sets involving many variables to highlight our approach. 

\smallskip
\noindent\textbf{Evaluation.} In terms of the evaluation metric, the median (50\%) and 75\%-percentile of the NMSE across these 10 experiments are reported. 
We choose to report median values instead of mean due to outliers (see box plots). This is a common practice for combinatorial optimization problems.

\noindent\textbf{Baselines.}  We consider the following baselines based on evolutionary algorithms: 1) Genetic Programming (GP)~\cite{DEAP_JMLR2012}.  2) Eureqa~\cite{DBLP:journals/gpem/Dubcakova11}. 
We also consider a series of baselines using reinforcement learning: 3) Priority queue training (PQT)~\cite{DBLP:journals/corr/abs-1801-03526}. 4) Vanilla Policy Gradient (VPG) that uses the REINFORCE algorithm~\cite{DBLP:journals/ml/Williams92} to train the model. 5) Deep Symbolic Regression (DSR)~\cite{DBLP:conf/iclr/PetersenLMSKK21}. 6) Neural-Guided Genetic Programming Population Seeding (GPMeld)~\cite{DBLP:conf/nips/MundhenkLGSFP21}.

We leave the detailed descriptions of the configurations of our \method~and baseline algorithms to {the supplementary materials and an anonymized website detailing our implementation\footnote{\url{https://github.com/jiangnanhugo/cvgp}}} and only mention a few implementation notes here. 
%
%
We implemented GP and \method. 
They use a data oracle, which returns (noisy) observations of the ground-truth equation when queried with inputs. 
We cannot implement the same Oracle for other baselines because of code complexity and/or no available code. 
To ensure fairness, the sizes of the training datasets we use for those baselines are larger than the total number of data points accessed in the full execution of those algorithms. 
In other words, their access to data would have no difference if the same oracle has been implemented for them because it does not affect the executions whether the data is generated ahead of the execution or on the fly. 
The reported NMSE scores in all charts and tables are based on separately generated data that have never been used in training.  
The threshold to freeze operands in \method is if the MSE to fit a data batch is below 0.01. The threshold to freeze the value of a constant in \method is if the variance of best-fitted values of the constant across trials drops below 0.001.

\subsection{Experimental Analysis}

\noindent\textbf{Learning Result.}
Our \method attains the smallest median (50\%) and 75\%-quantile NMSE values among all the baselines mentioned in Section~\ref{sec:exp-set}, when evaluated on noiseless datasets (Table~\ref{tab:summary-nmse}) and noisy datasets (Table~\ref{tab:summary-noisy}). This shows our method can better handle multiple variables symbolic regression problems than the current best algorithms in this area.

\noindent\textbf{Ablation Studies.}
We use box plots in Figure~\ref{fig:evalucate-metric}(a-d) to show that the superiority of our \method~generalizes to other quantiles beyond the 50\% and 75\%-quantile. 
We also show the performance is consistent under the variations of evaluation metrics in Figure~\ref{fig:evalucate-metric}(a-d), and noise levels in Figure~\ref{fig:evalucate-metric}(e-f).

\noindent\textbf{Recovering Ground-truth Equations.} For  relatively less challenging noiseless datasets (\textit{i.e.}, $(2,1,1)$ with various operand sets), our \method~sometimes recovers ground-truth expressions. 
We evaluate the percentage that each algorithm successfully detects the ground-truth expressions on  $50$ randomly generated  benchmark datasets. Table~\ref{tab:recovery} shows that our \method  algorithm has a higher chance to recover ground-truth expressions than the GP method. 

\begin{table}[!t]
    \centering
    \begin{tabular}{c|c|cc}\hline
        Operand set &Dataset configs&  \method (ours) & GP\\  \hline
        $\{\texttt{inv},+,-,\times\}$  &\multirow{3}{*}{(2,1,1) }& $\mathbf{64\%}$ & $44\%$  \\
        $\{\sin,\cos,+,-,\times\}$ & & $\mathbf{46\%}$ & $22\%$  \\
        $\{\sin, \cos, \texttt{inv},+,-,\times\}$ & & $\mathbf{44\%}$ & $32\%$  \\
\hline
    \end{tabular}
    \caption{Our \method~has a higher rate to recover the ground-truth expressions compared to GP on 3 simple datasets.}
    \label{tab:recovery}
\end{table}

\section{Conclusion}
In this research, we propose Control Variable Genetic Programming (\method) for symbolic regression with many independent variables. This  is beyond current state-of-the-art approaches mostly tested on equations with one or two variables. 
\method~builds equations involving more and more independent variables via control variable experimentation. 
Theoretically, we show \method~as an incremental building approach can bring an exponential reduction in the search spaces when learning a class of expressions. 
In experiments, \method finds the best-fitted expressions  among 7 competing approaches and on dozens of benchmarks. 

\section{Acknowledgments}
This research was supported by NSF grants IIS-1850243, CCF-1918327.

\bibliography{reference}
\bibliographystyle{unsrtnat}

\newpage
\appendix
\section{Experiment Settings} \label{apx:exp-set}

\subsection{Dataset Configuration}

 \noindent\textbf{Synthesised datasets}
 We generated several families of multiple-variable ground-truth expressions, and use these expressions to generate the datasets for control variable experiments $\{(\mathbf{x}_{i}, y_i)\}_{i=1}^n$.  
We label the datasets by 1) the set of mathematical operators that can be included in the ground-truth expression that generates the dataset; 2) the number of independent variables $m$; 3) the number of single terms of the ground-truth expression; 4) the number of cross terms of the ground-truth expression. We write items 2), 3), and 4) into a tuple in various charts and tables to represent the dataset. 

These four characteristics of the ground-truth expression determine the complexity of learning the ground-truth expression from the dataset. 
To give an example, one ground-truth expression that generates a dataset labeled by the ``\texttt{inv}, $+, -, \times$'' operators with configuration $(2,1,1)$ can be:
\begin{equation*}
0.4967-\frac{0.6824}{x_1}-\frac{0.7346 x_1}{x_0}
\end{equation*}
This expression contains two variables $x_1,x_2$, a cross term $\frac{x_1}{x_0}$ with a  constant  coefficient $0.7346$, a single term $1/x_1$ with a  constant  coefficient $0.6824$, and a constant $0.4967$. We give more examples of such ground-truth expressions in Table~\ref{tab:apx-dataset-example}.

\begin{table}[!ht]
    \centering
    \begin{tabular}{c|c}
    \toprule
      Dataset Configs &  \multirow{1}{*}{Example  expression}  \\  \hline
   (2,1,1) & $0.497-{0.682}/{x_1}-{0.735 x_1}/{x_0}$\\
    (3,2,2) & $-0.603 x_0x_1 + 0.744 x_0 + {0.09 x_1}/{x_2} + 0.562 + {0.582}/{x_2}$ \\
    \hline
 \multicolumn{2}{c}{\textbf{(a)} Datasets containing operands $\{\texttt{inv},+,-,\times\}$}. \\ 
         \hline
 (2,1,1) & $0.259 x_0 \sin(x_1) + 0.197 x_1 - 0.750$\\
 (3,2,2) & $-0.095 x_0 x_2 + 0.012 x_2 \sin(x_1) - 0.576 x_2 - 0.214 \cos(x_0) - 0.625$\\
 \hline
  \multicolumn{2}{c}{\textbf{(b)} Datasets containing operands $\{\sin,\cos,+,-,\times\}$}. \\ 
        \hline
(2,1,1) & $0.7272\sin(x_0)-0.3866+{0.183}/{x_0}$ \\
(3,2,2)& ${0.7167 x_0}/{x_2}- 0.0632 x_1 + 0.2746 x_2 \cos(x_1) - 0.7293 +{0.0627}/{x_2}$ \\
\hline
 \multicolumn{2}{c}{\textbf{(c)} Datasets containing operands $\{\sin,\cos,\texttt{inv},+,-,\times\}$}. \\ 
    \end{tabular}
    \caption{Example expressions used in our experiments with different dataset configurations and the set of operands.}
    \label{tab:apx-dataset-example}
\end{table}

\noindent\textbf{Noisy Dataset Setting.} In real scientific experiments, the datasets often contain noises.
We add Gaussian noise $\mathcal{N}(0, \sigma^2)$ to the output $y$ in the dataset and control the noise rate by varying the values of $\sigma$ in $\{0.02,0.04,0.08,0.1, 0.12, 0.14\}$.

\subsection{ Evaluation Metrics}
Given a dataset $\{(\mathbf{x}_{i},y_i)\}_{i=1}^n$ generated from the ground-truth expression $\phi$, where $n$ indicates the number of test samples. 
The empirical variance of the target values $\sigma_y$ is defined as:
\begin{equation}
\sigma_y^2=\frac{1}{n}\sum_{i=1}^n \left(y_i-\frac{1}{n}{\sum_{i=1}^n y_i}\right)^2
\end{equation}
During training and testing, we measure the goodness-of-fit of a candidate expression $\bar{\phi}$,  by evaluating the following  Mean-square error (MSE), Negative Mean-square error (MSE), normalized  Mean-square error (NMSE), normalized root Mean-squared error (NRMSE), Invese normalized root Mean-squared error (InvNRMSE): 
\begin{equation}\label{eq:inv-nrmse}
\begin{aligned}
\text{MSE}&=\frac{1}{n}\sum_{i=1}^n(y_{i}-\bar{\phi}(\mathbf{x}_{i}))^2, \\
\text{NegMSE}&=-\frac{1}{n}\sum_{i=1}^n(y_{i}-\bar{\phi}(\mathbf{x}_{i}))^2, \\
\text{NMSE}&=\frac{1}{n}\frac{\sum_{i=1}^n(y_{i}-\bar{\phi}(\mathbf{x}_{i}))^2}{\sigma_y^2},\\ 
\text{RMSE}&=\sqrt{\frac{1}{n}\sum_{i=1}^n(y_{i}-\bar{\phi}(\mathbf{x}_{i}))^2} ,\\
\text{NRMSE}&=\frac{1}{\sigma_y} \sqrt{\frac{1}{n}\sum_{i=1}^n(y_{i}-\bar{\phi}(\mathbf{x}_{i}))^2} \\
\text{InvNRMSE}&=\frac{1}{\frac{1}{\sigma_y} \sqrt{\frac{1}{n}\sum_{i=1}^n(y_{i}-\bar{\phi}(\mathbf{x}_{i}))^2}}
\end{aligned}
\end{equation}

\subsection{Baselines implementation}

\noindent\textbf{GP} The implementation is based on the deap package\footnote{\url{https://github.com/DEAP/deap}}. 
However, we re-implemented the code following the concept of their package. 

\noindent\textbf{\method} Our method is implemented on top of the \textbf{GP} following Algorithm~\ref{alg:cvgp}. 

\noindent\textbf{Eureqa} This algorithm is currently maintained by the DataRobot webiste\footnote{\url{https://docs.datarobot.com/en/docs/modeling/analyze-models/describe/eureqa.html}}. We use the python API provided \footnote{\url{https://pypi.org/project/datarobot/}} to send the training dataset to the DataRobot website and collect the predicted expression after 30 minutes. This website only allows us to execute their program under a limited budget. Due to budgetary constraints, we were only able to test the datasets for the noiseless settings. 
For the Eureqa method, the fitness measure function is negative RMSE. We generated large datasets of size $10^5$ in training each benchmark.

\noindent\textbf{DSR, PQT, GPMeld} These algorithms are evaluated based on an implementation in  \footnote{\url{https://github.com/brendenpetersen/deep-symbolic-optimization}}. For every ground-truth expression, we generate a dataset of sizes $10^5$ training samples. Then we execute all these baselines on the dataset with the  configurations listed in Table~\ref{tab:apx-configuration}.
For the four baselines (\textit{i.e.}, PQT, VPG, DSR, GPMeld), the reward function is INV-NRMSE, which is defined as $\frac{1}{1+\text{NRMSE}}$. 

Note that Wolfram was not considered in this research, because the current ``\texttt{FindFormula}'' function in the Wolfram language only supports searching for single variable expressions.

\begin{table}[!t]
    \centering
    \scalebox{0.95}{
    \begin{tabular}{r|cccccc} \toprule
          & \method & GP&  DSR & PQT & GPMeld &Eureqa\\\midrule
         Reward function & NegMSE &NegMSE & InvNRMSE & InvNRMSE & InvNRMSE & NegRMSE\\ 
        Training set size & $25,600$ & $25,600$ & $$ 50, 000$$ & $ 50, 000$ & $ 50, 000$ & $ 50, 000$ \\
         Testing set size & $256$ & $25,600$ & $$ 256$$ & $256$ & $ 256$ & $ 256$ \\
         Batch size & $256$ & $256$ & $1024$ & $1024$ & $1024$ & $N/A$\\
      \#CPUs for training& 1 & 1 & 4 & 4 & 4 & 1\\
  $\epsilon$-risk-seeking policy  & N/A  &  0.02 & 0.02 & 0.02 & N/A& N/A\\ \midrule
        \#generic generations & 100 & 100 & N/A & N/A  & 60  &10,000\\
        \#Hall of fame &  10 & 10 & 25& 25& 25 & N/A \\
        Mutation Probability & {0.5} & 0.5 & 0.5 & N/A &N/A &N/A \\
        Mating Probability & {0.5} & 0.5 & 0.5 & N/A &N/A &N/A  \\
        \midrule
        training time (hours) & $\sim$0.5 & $\sim$0.5  & $\sim$0.5  & $\sim$0.5 & $\sim$6 & $\sim$0.5  \\
        \bottomrule
    \end{tabular}}
    \caption{Major hyper-parameters settings for all the algorithms considered in the experiment.}
    \label{tab:apx-configuration}
\end{table}

\subsection{Hyper-parameter Configurations}

We list the major hyper-parameter setting for all the algorithms in Table~\ref{tab:apx-configuration}. Note that if we use the default parameter settings, the GPMeld algorithm takes more than 1 day to train on one dataset. 
Because of such slow performance, we cut the  number of genetic programming generations in GPMeld by half to ensure fair comparisons with other approaches.

\section{Extended Experimental Analysis}  \label{apx:extend-experiment}

\begin{figure}[!ht]
    \centering
     \includegraphics[width=0.3\linewidth]{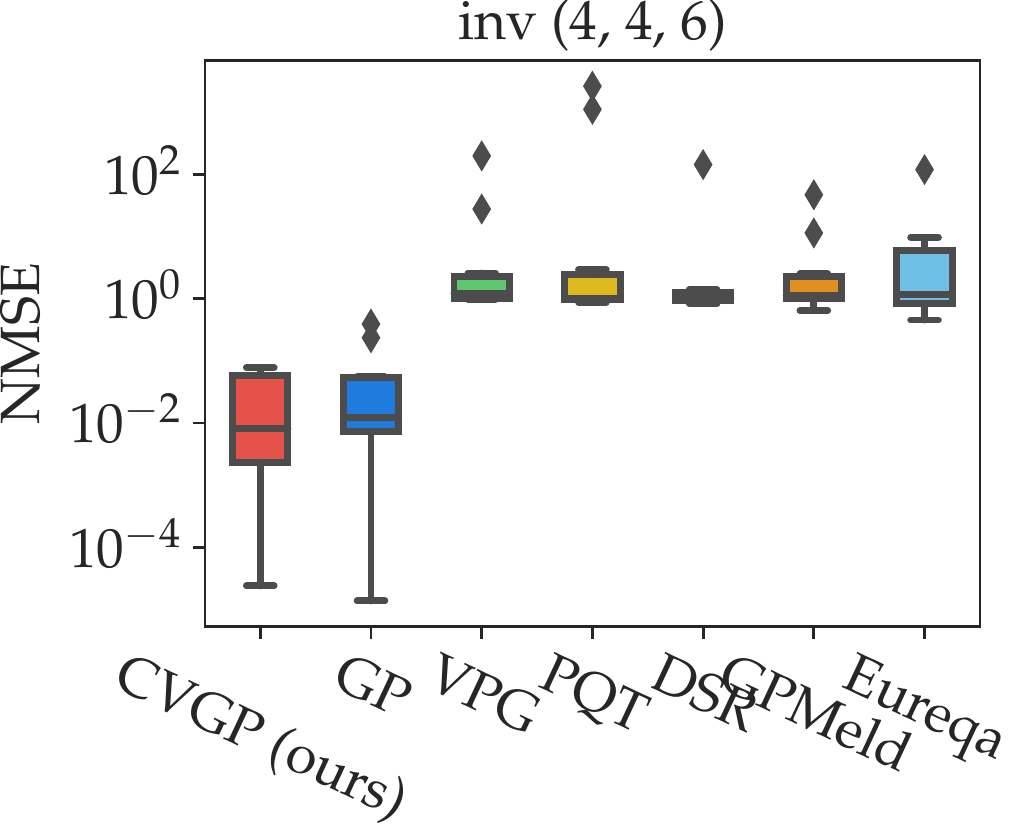}
    \includegraphics[width=0.3\linewidth]{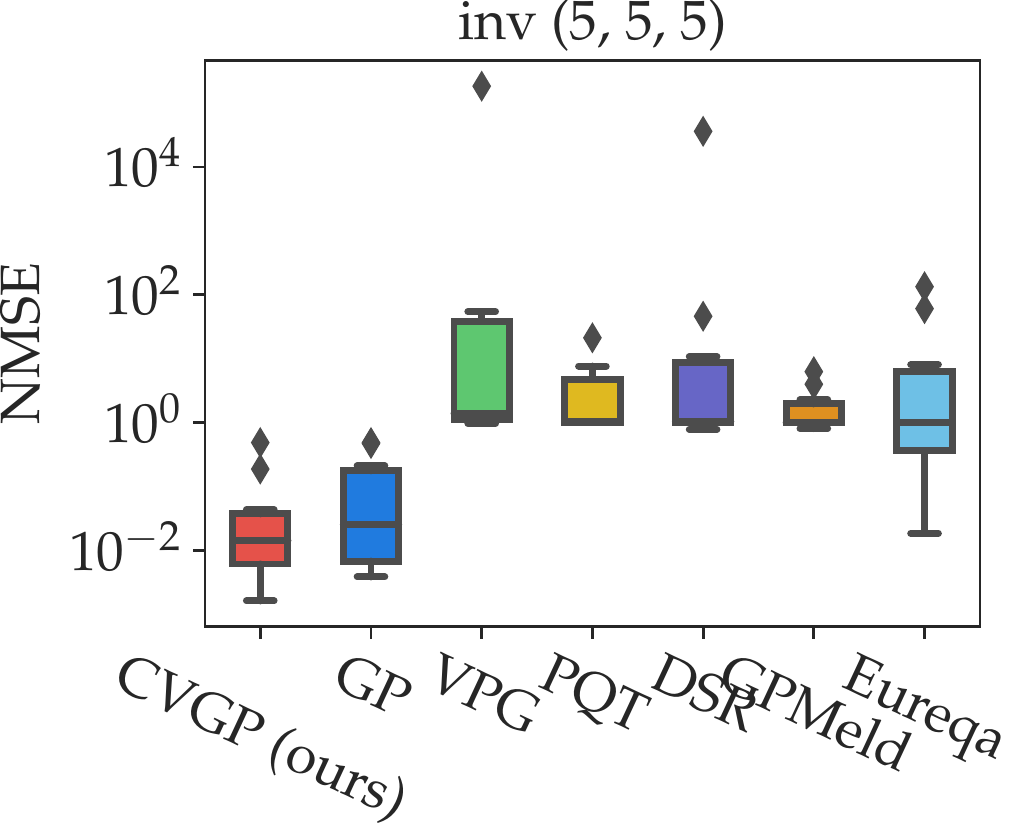}
    \includegraphics[width=0.3\linewidth]{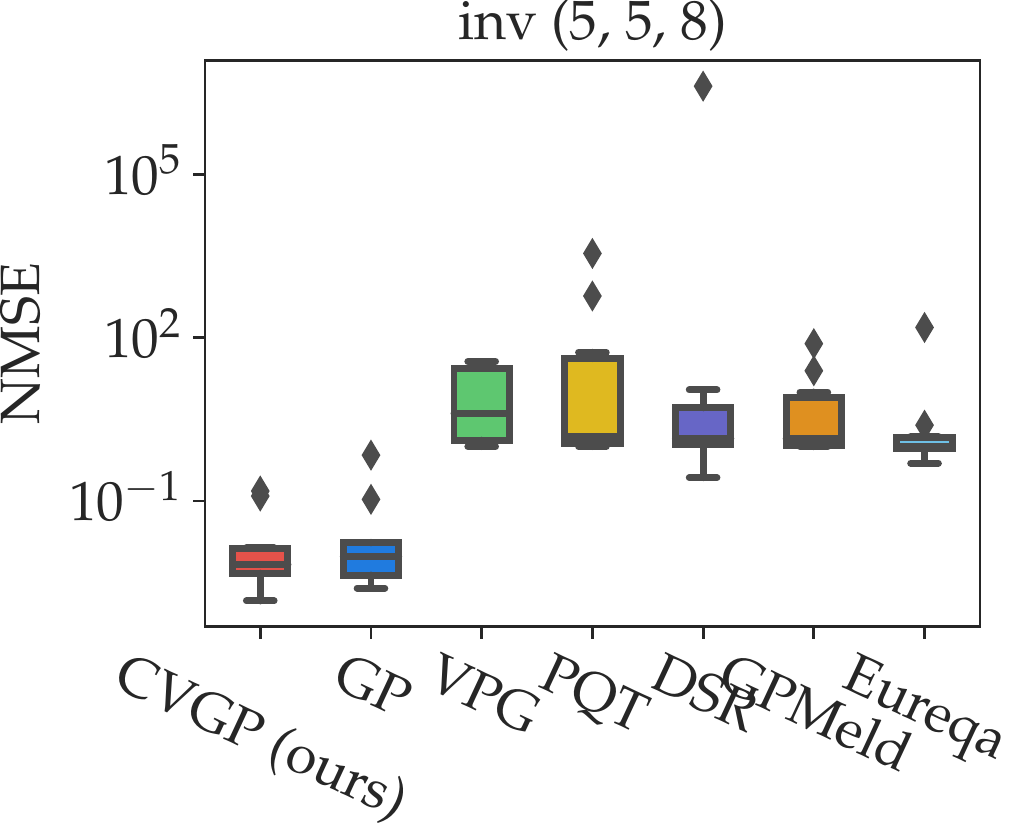}
    \includegraphics[width=0.3\linewidth]{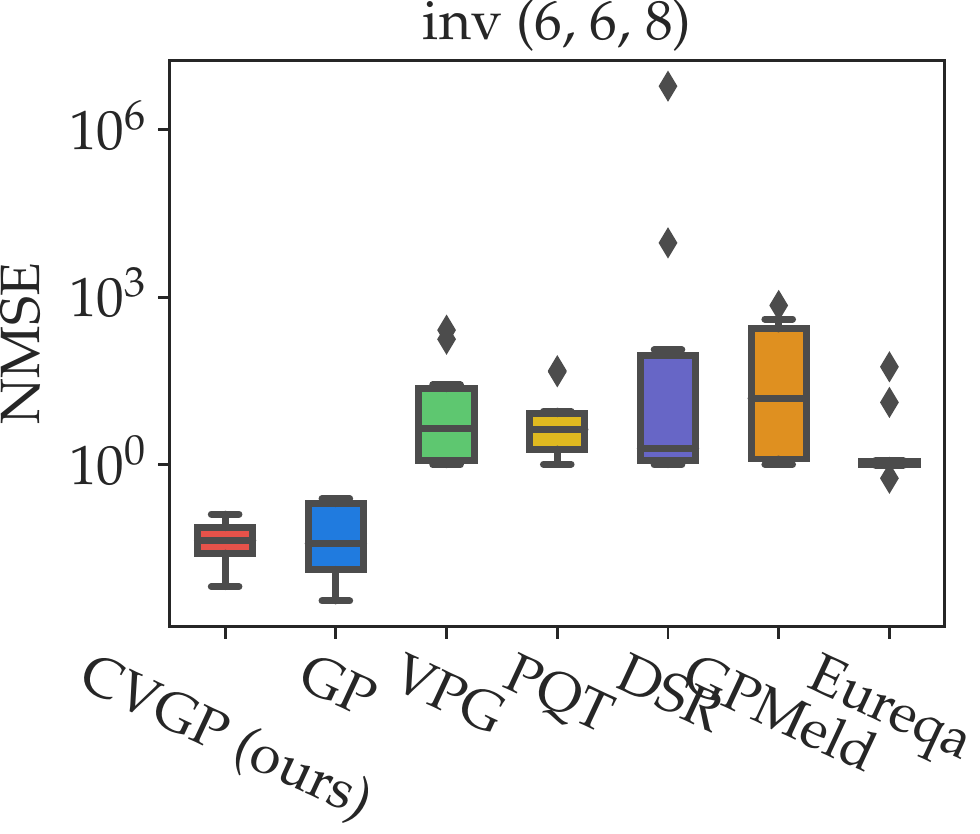}
    \includegraphics[width=0.3\linewidth]{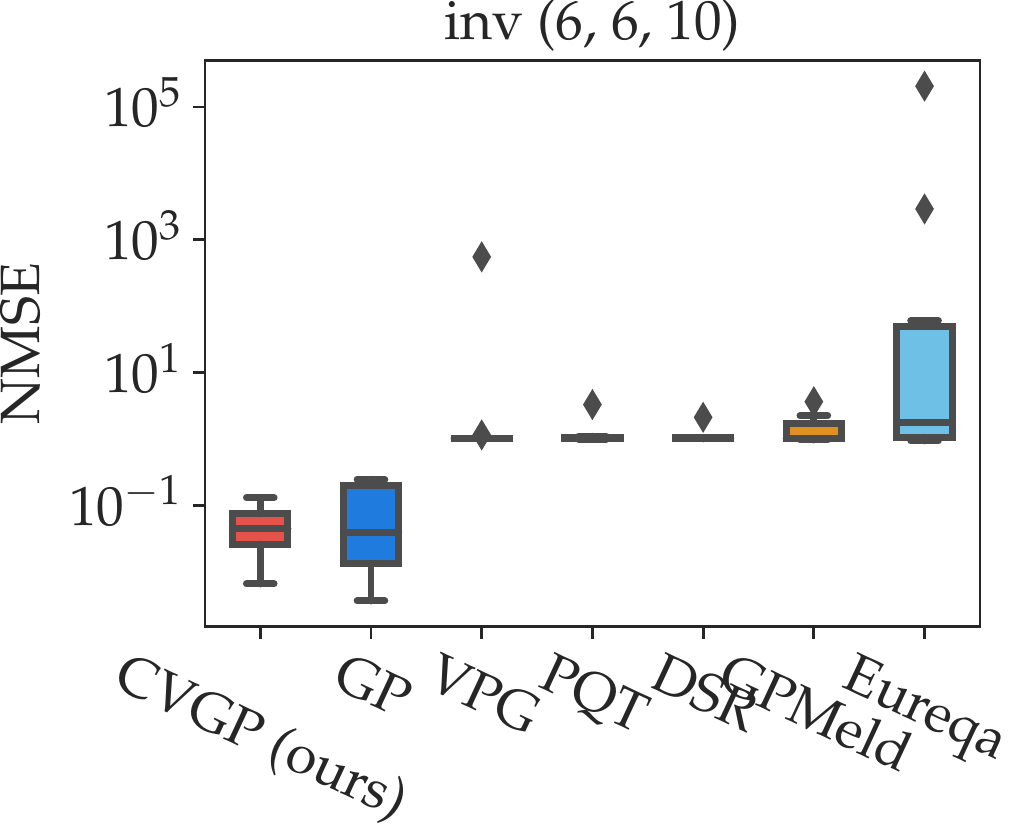}
    \includegraphics[width=0.3\linewidth]{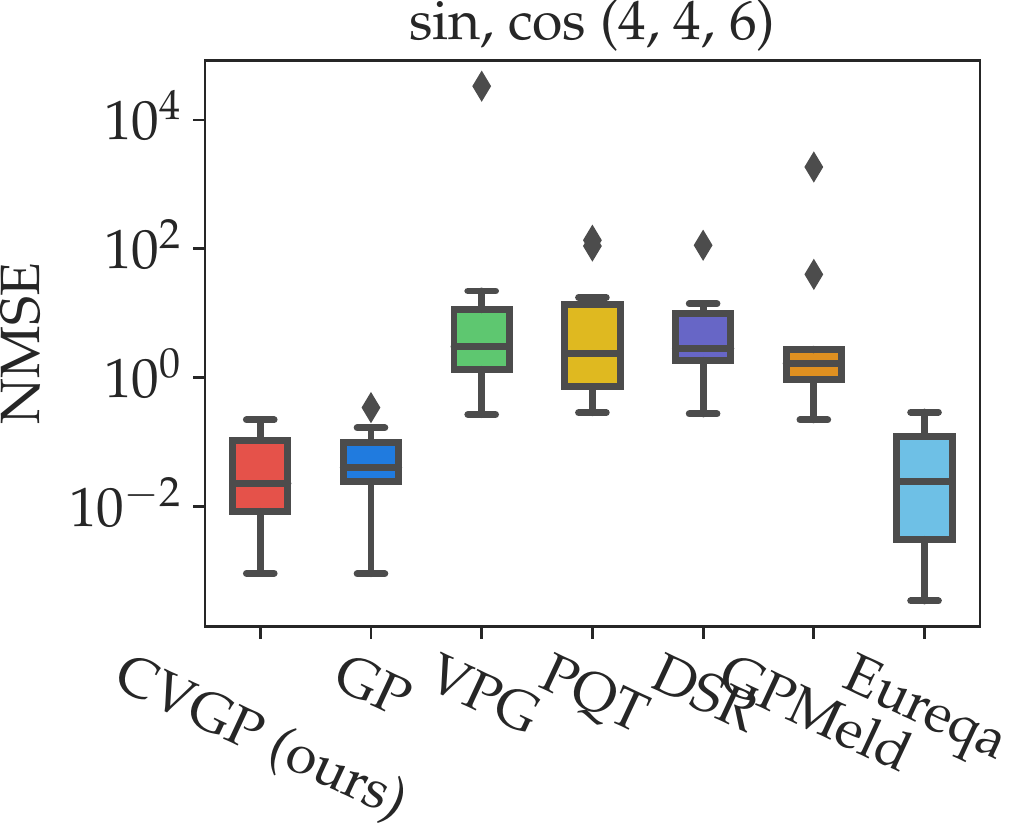}
    \includegraphics[width=0.3\linewidth]{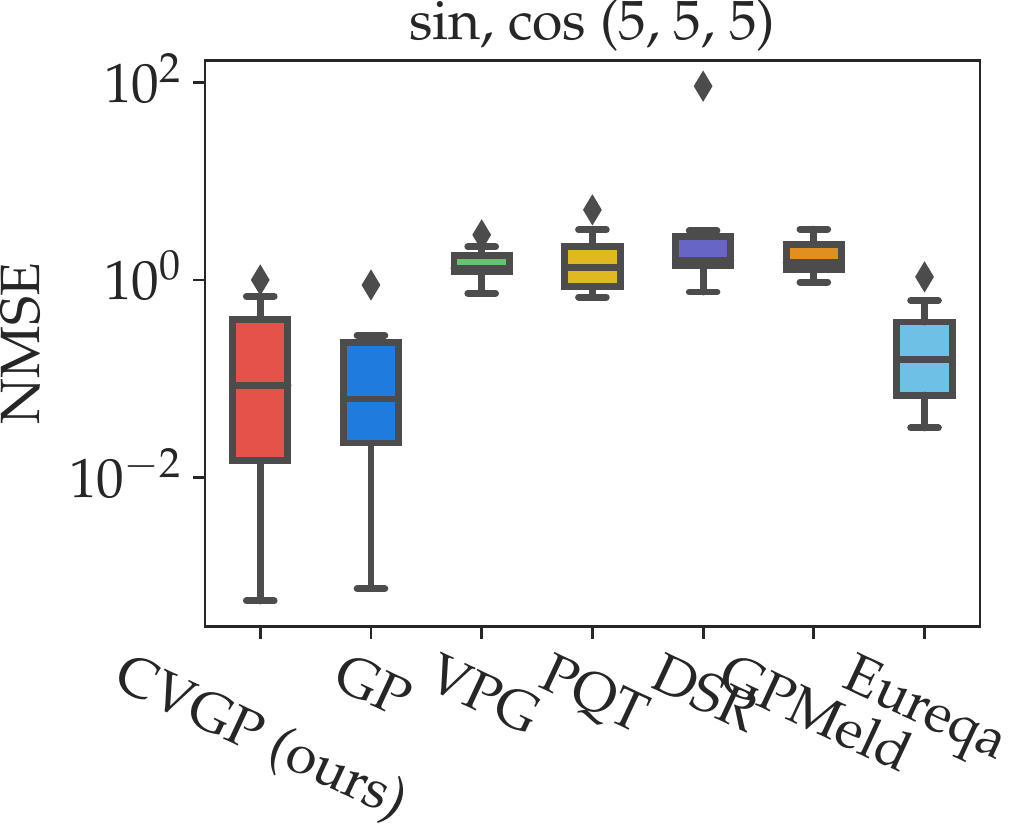}
    \includegraphics[width=0.3\linewidth]{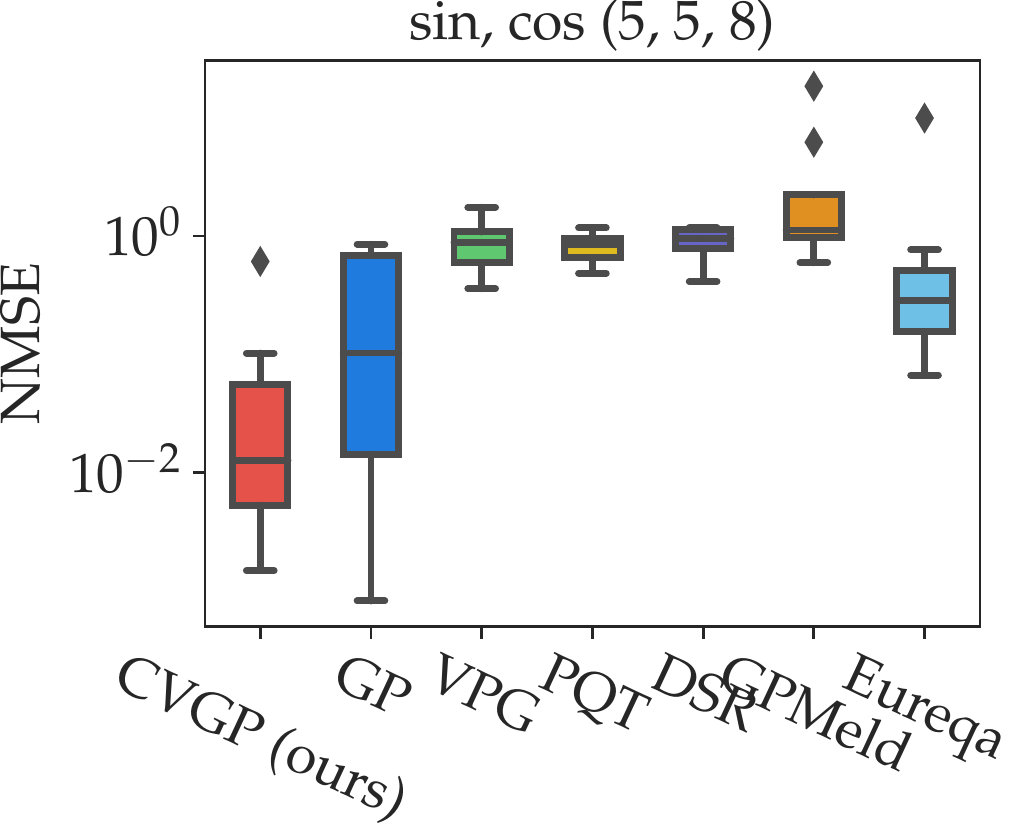}
    \includegraphics[width=0.3\linewidth]{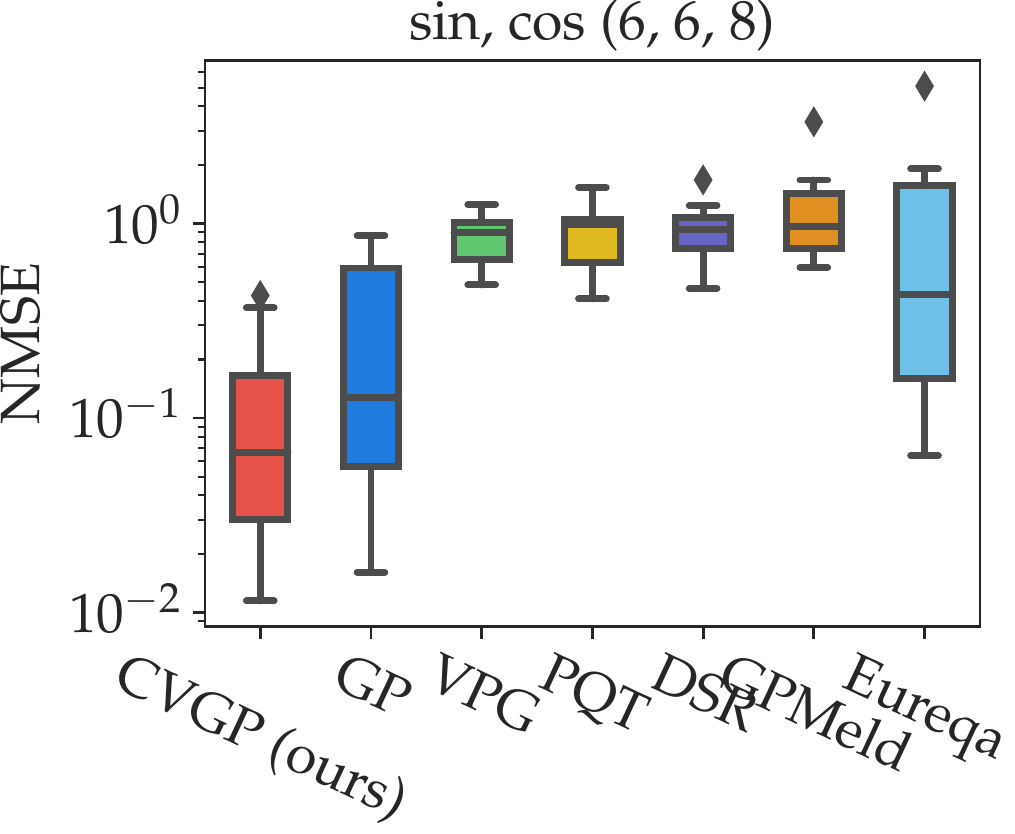}
    \includegraphics[width=0.3\linewidth]{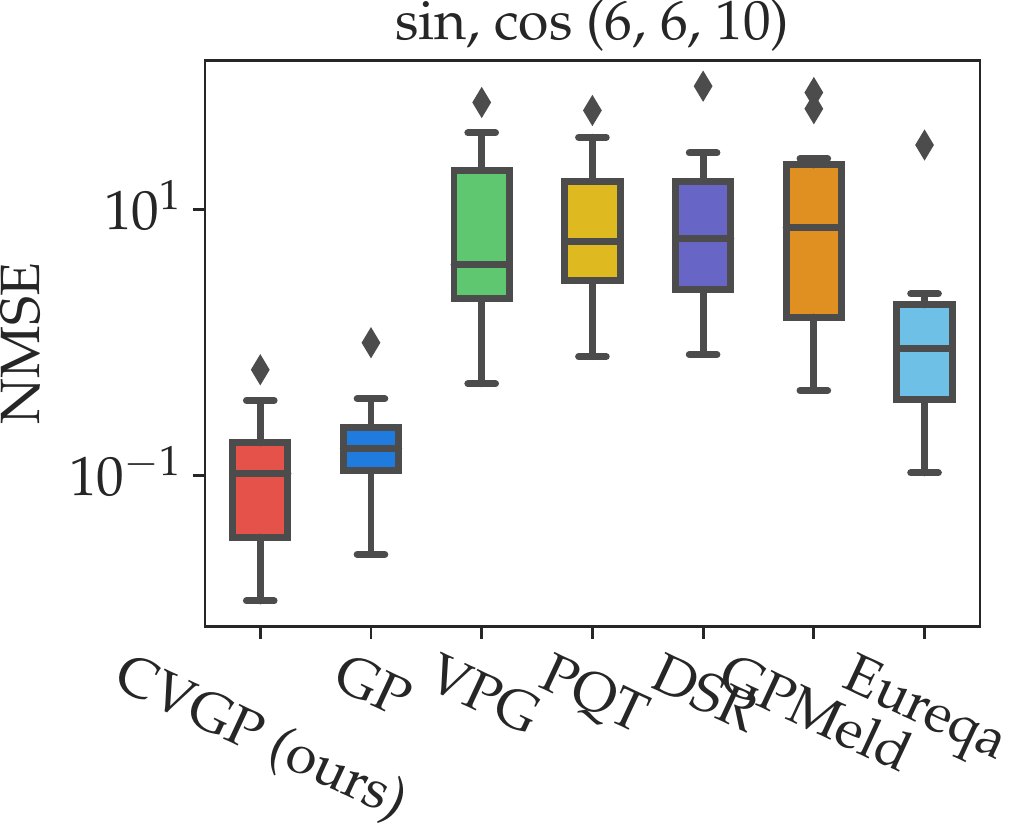}
    \includegraphics[width=0.3\linewidth]{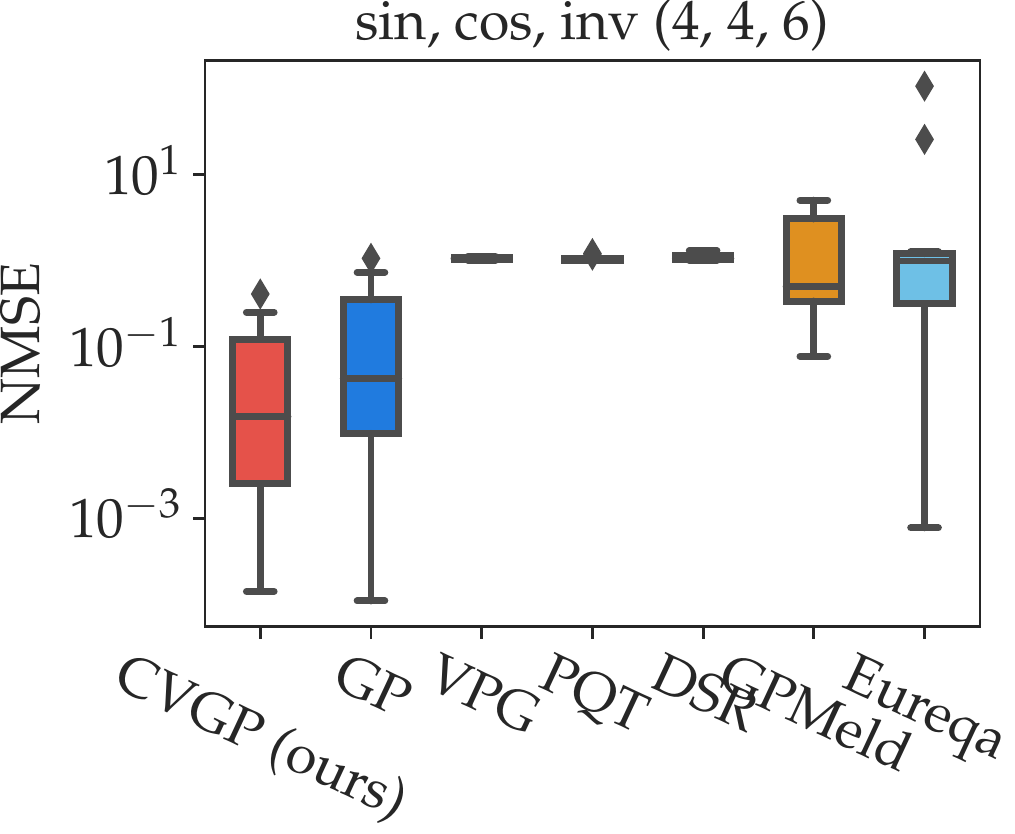}
    \includegraphics[width=0.3\linewidth]{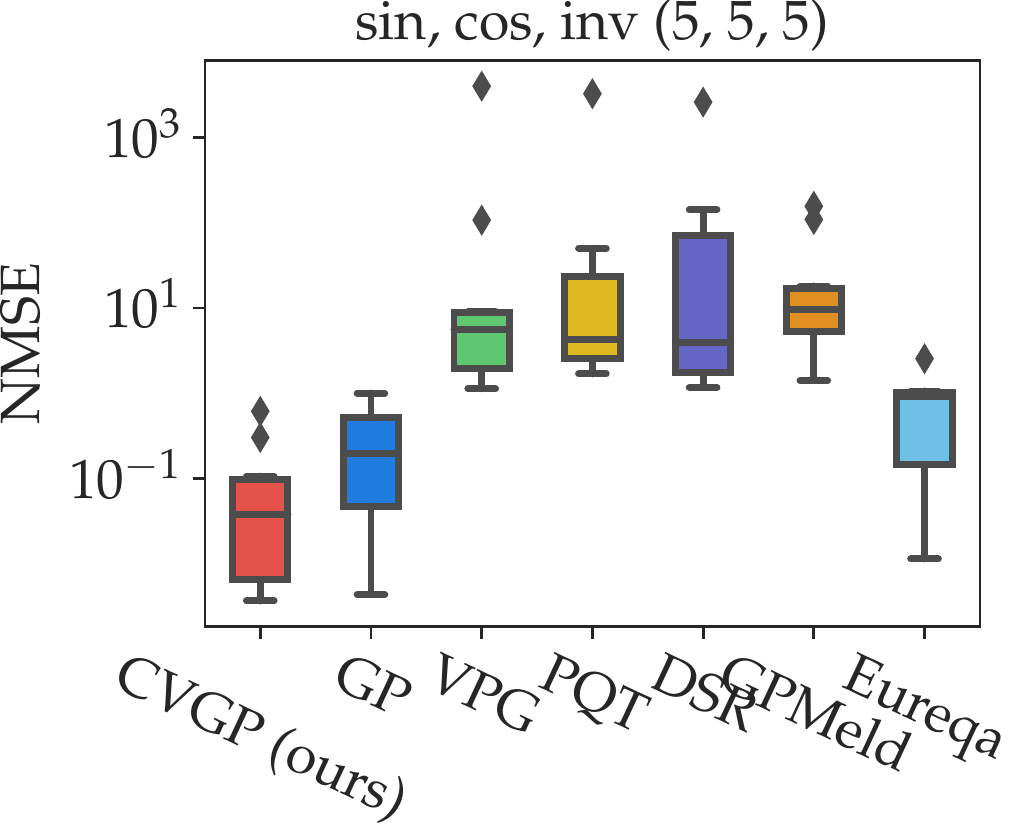}
    \includegraphics[width=0.3\linewidth]{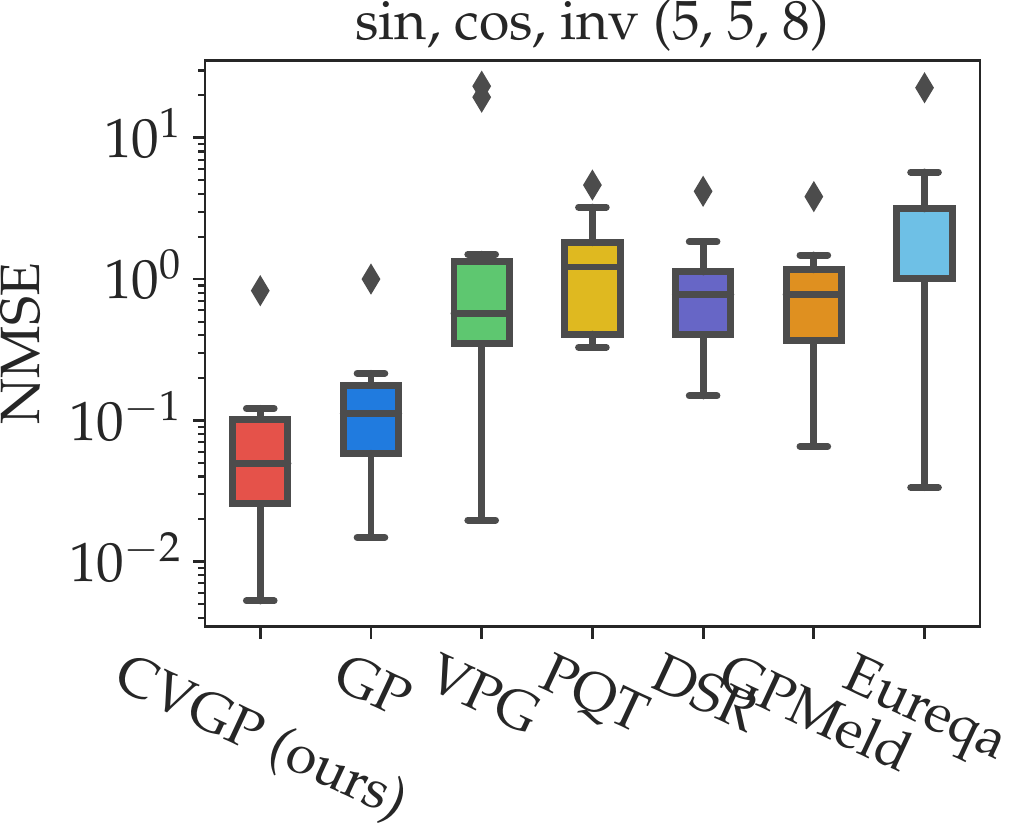}
    \includegraphics[width=0.3\linewidth]{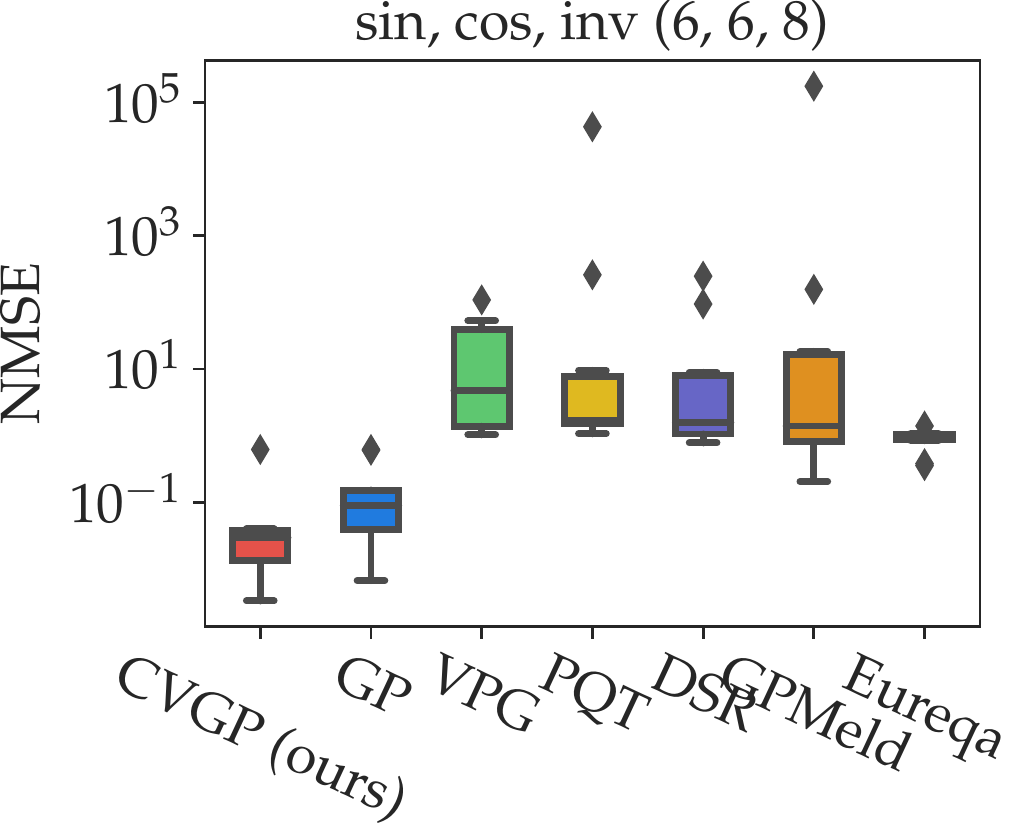}
    \includegraphics[width=0.3\linewidth]{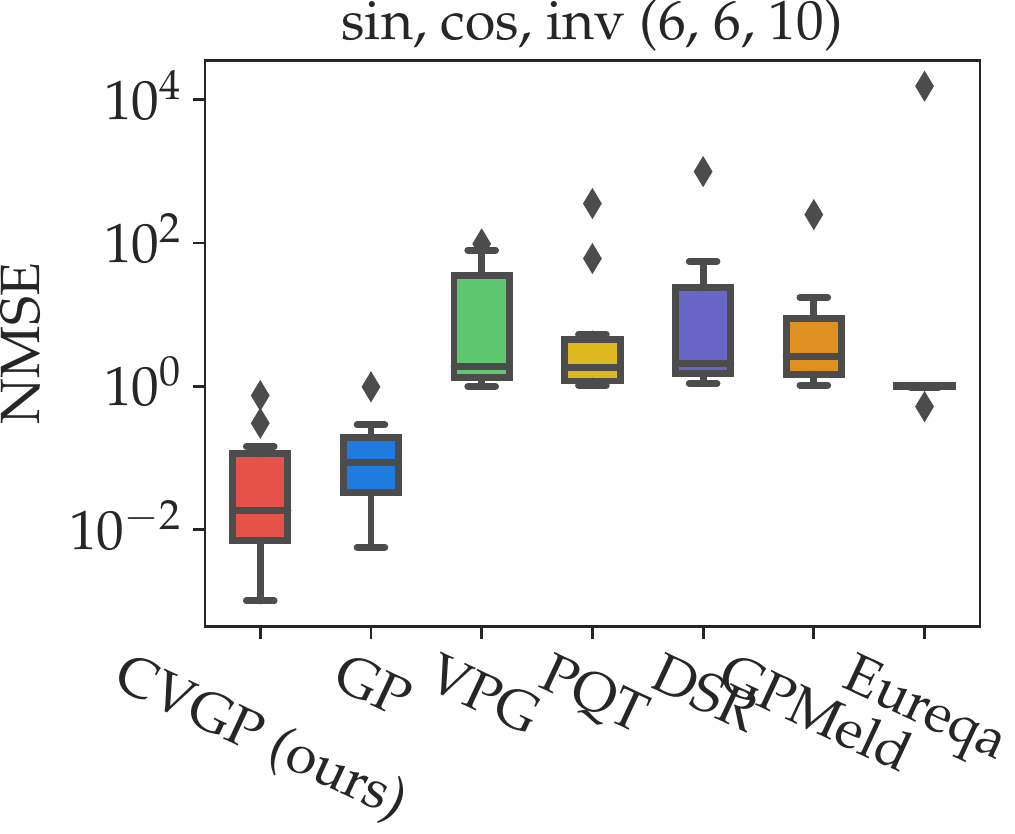}
    \caption{Quartiles of NMSE values of all the methods over several \textit{noiseless} datasets. Our \method shows a consistent improvement over all the baselines considered, among all the datasets. }
    \label{fig:Quartile-nmse-noiseless-full}
\end{figure}

\begin{figure}[!ht]
    \centering
     \includegraphics[width=0.3\linewidth]{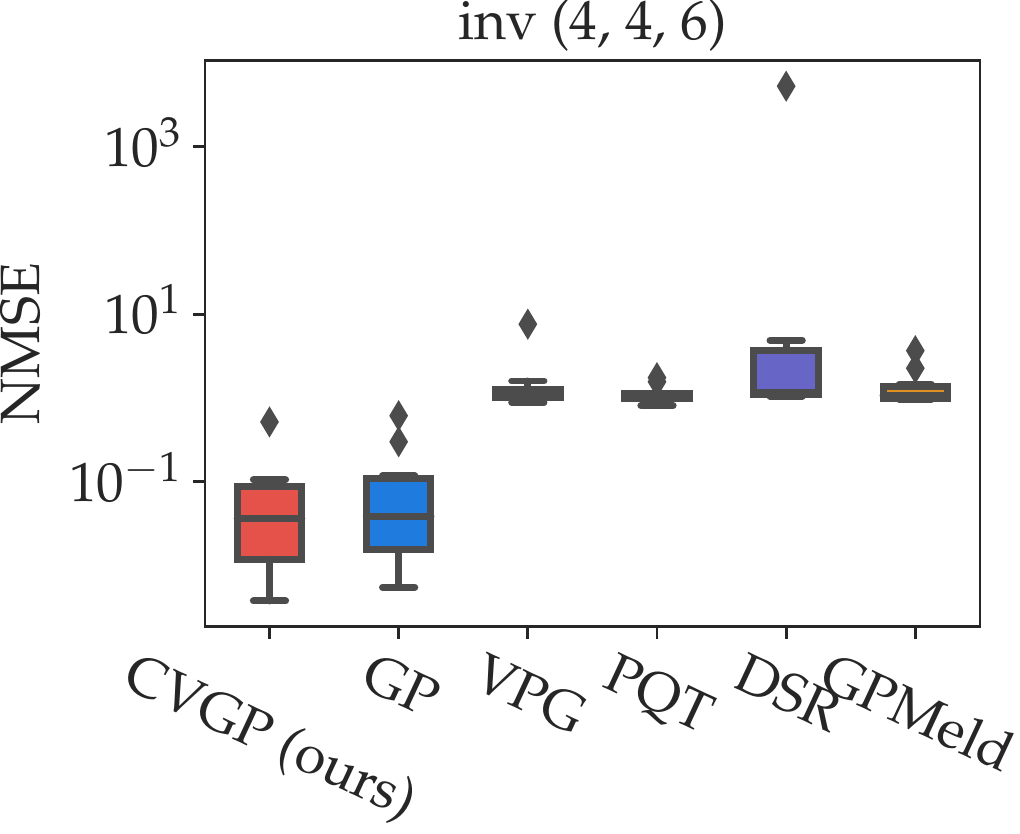}
    \includegraphics[width=0.3\linewidth]{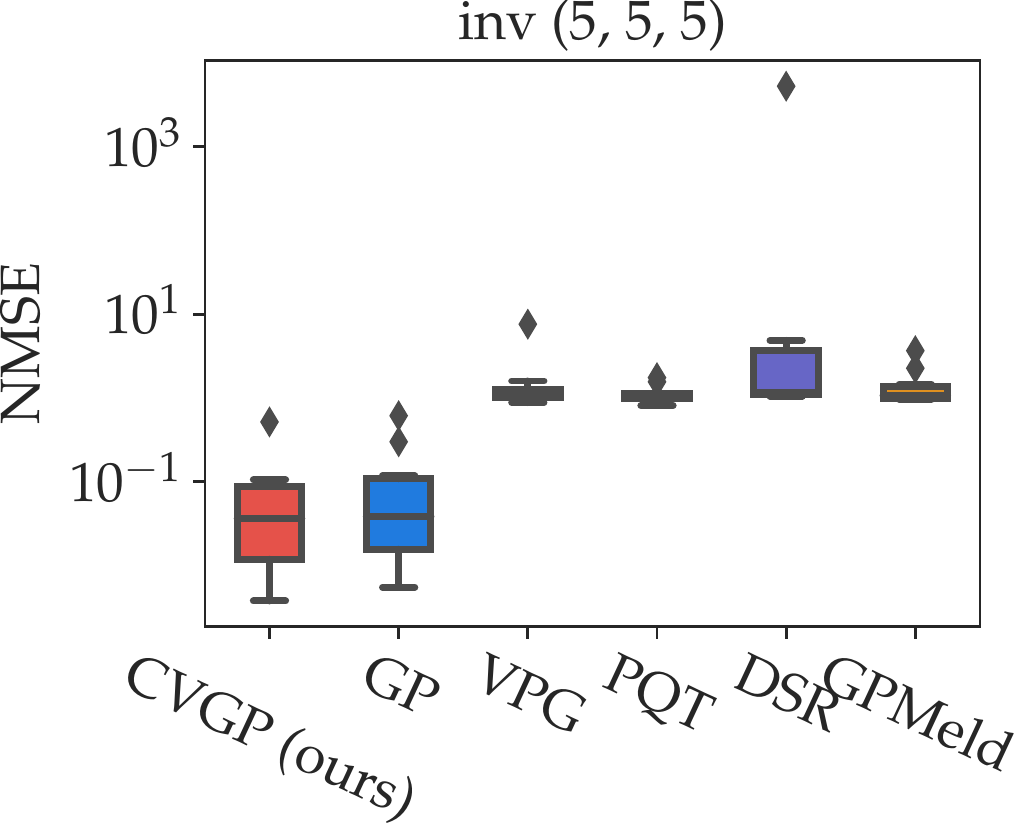}
    \includegraphics[width=0.3\linewidth]{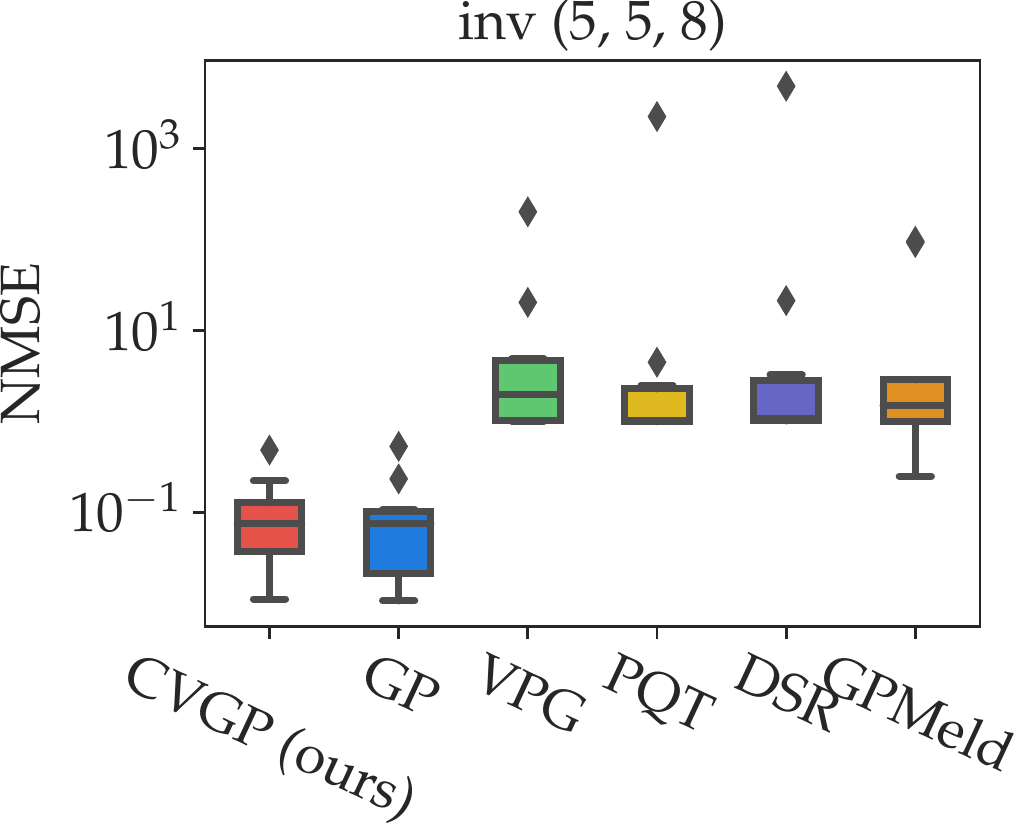}
    \includegraphics[width=0.3\linewidth]{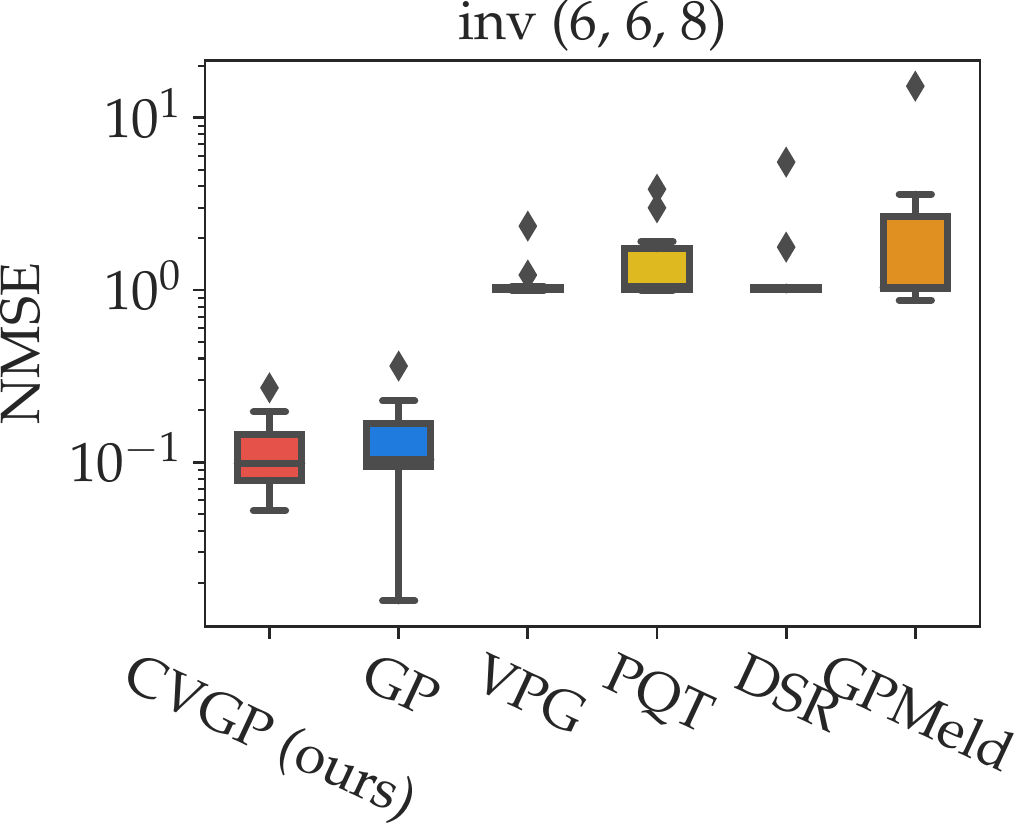}
    \includegraphics[width=0.3\linewidth]{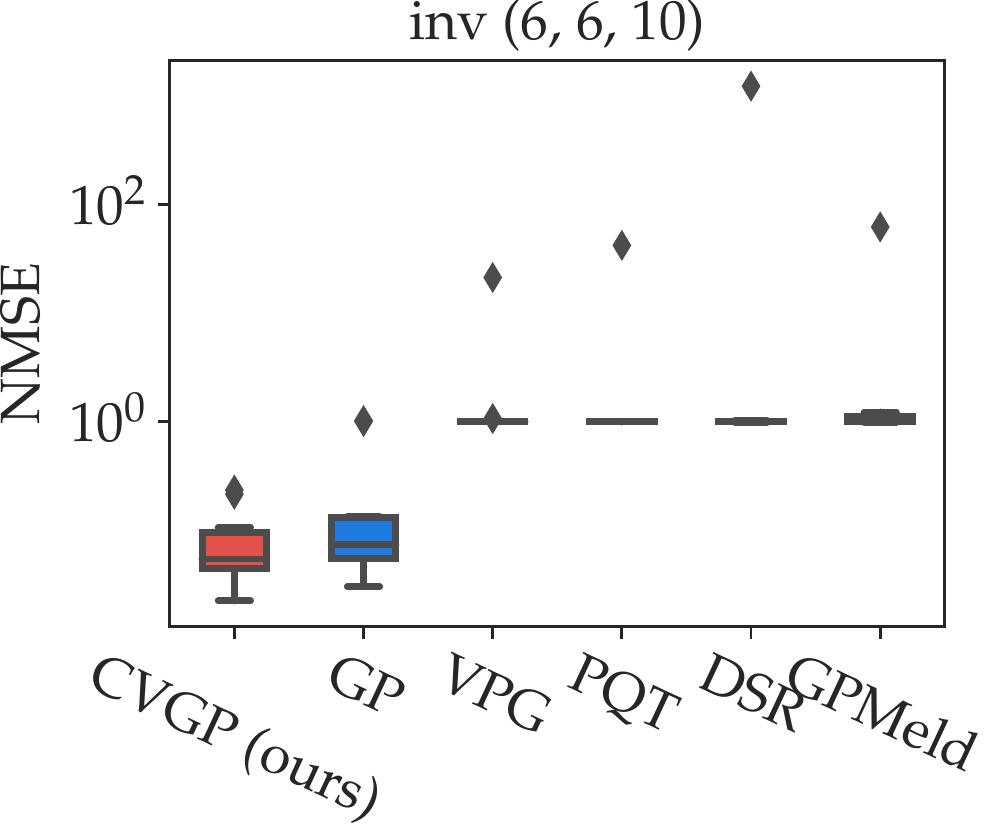}
    \includegraphics[width=0.3\linewidth]{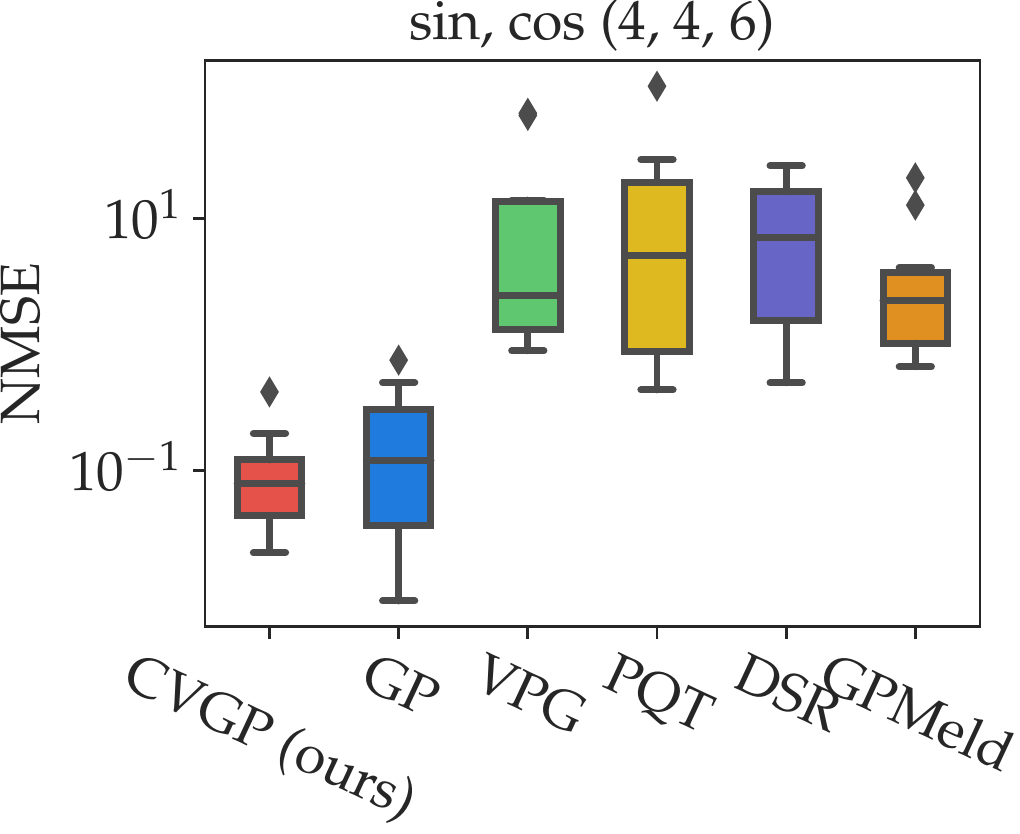}
    \includegraphics[width=0.3\linewidth]{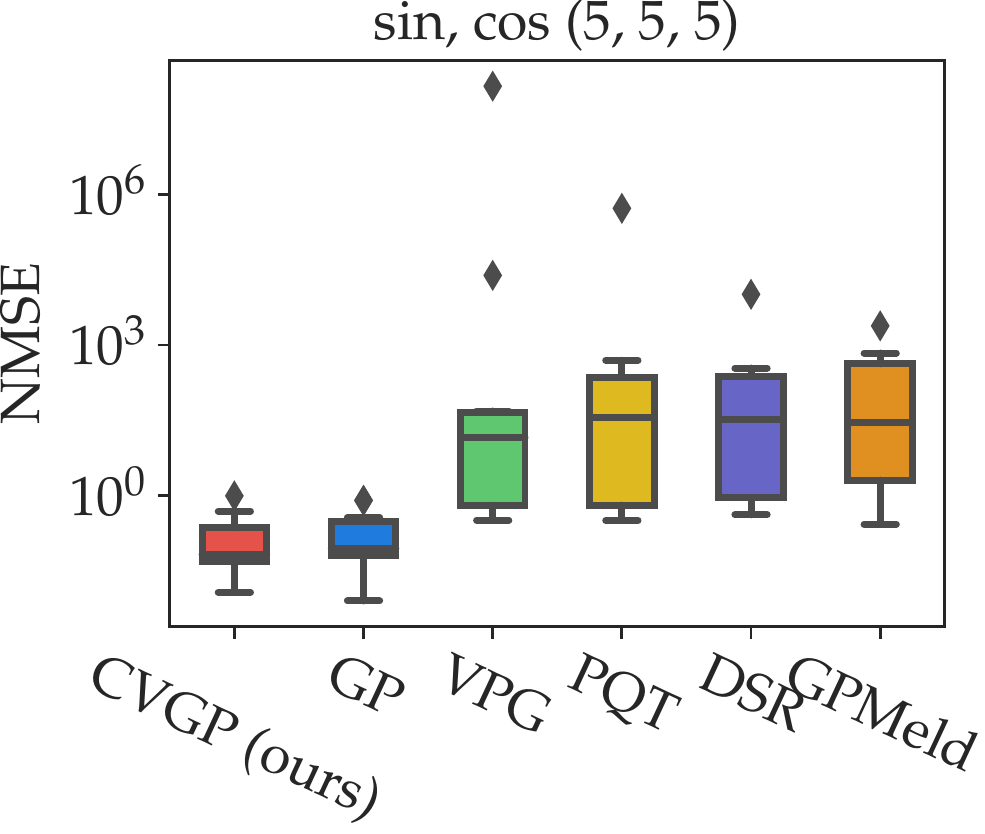}
    \includegraphics[width=0.3\linewidth]{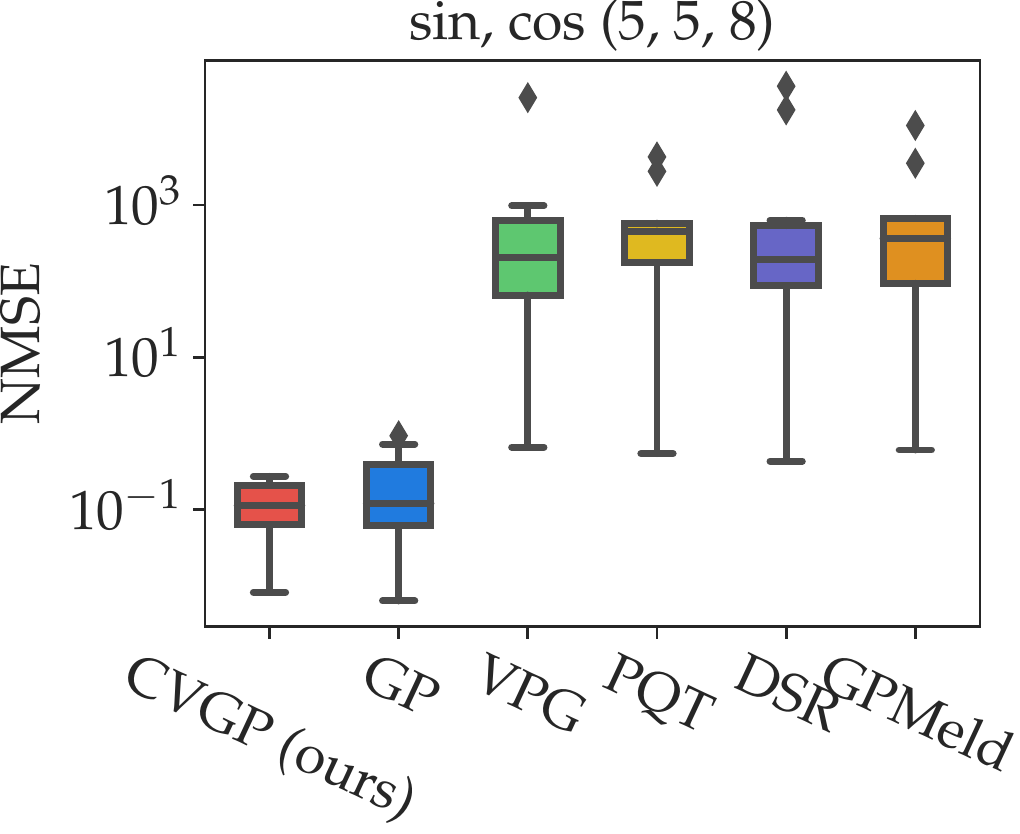}
    \includegraphics[width=0.3\linewidth]{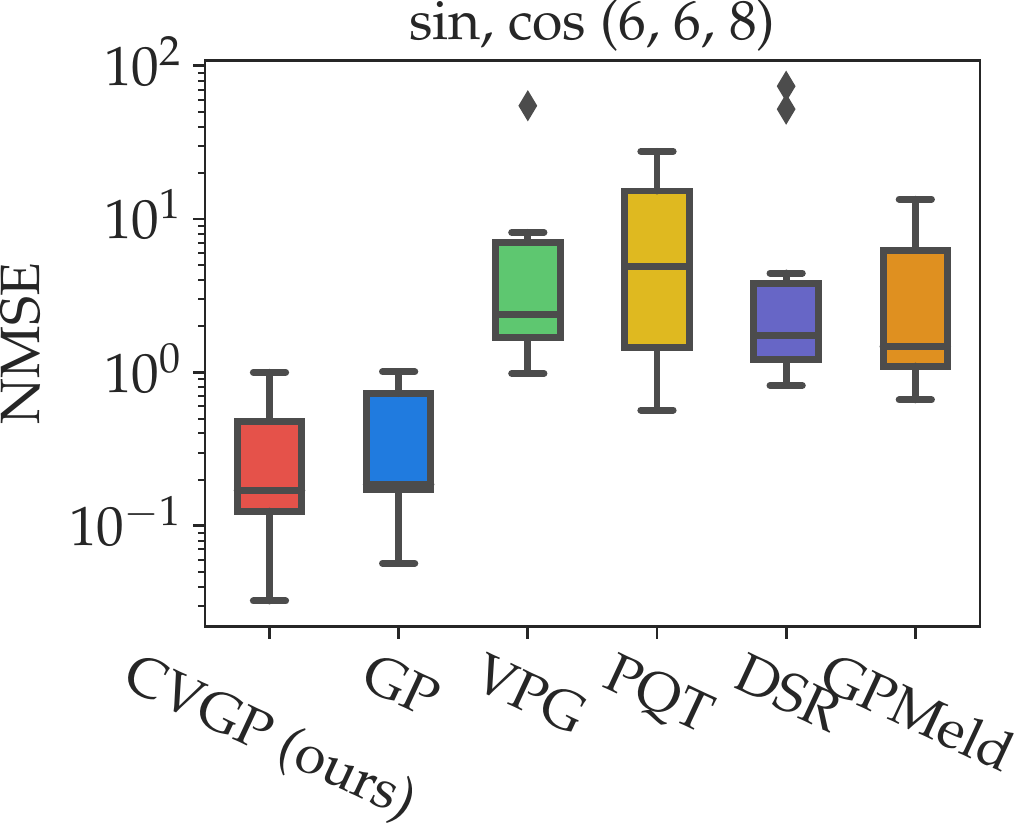}
    \includegraphics[width=0.3\linewidth]{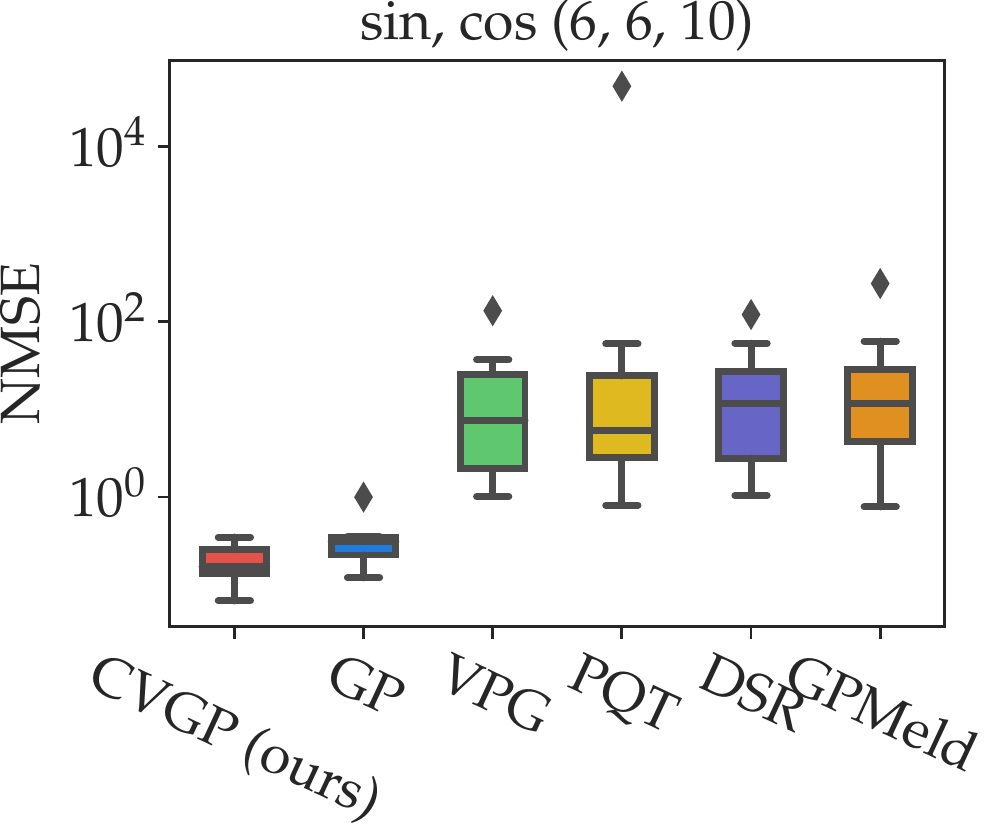}
    \includegraphics[width=0.3\linewidth]{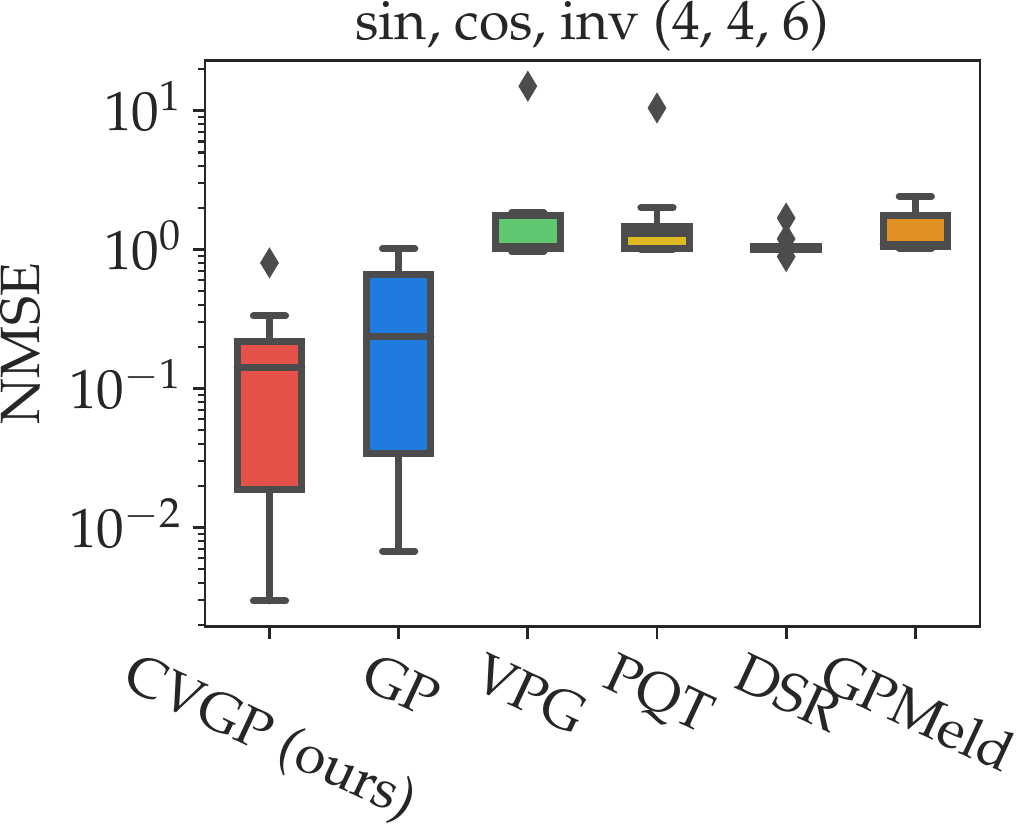}
    \includegraphics[width=0.3\linewidth]{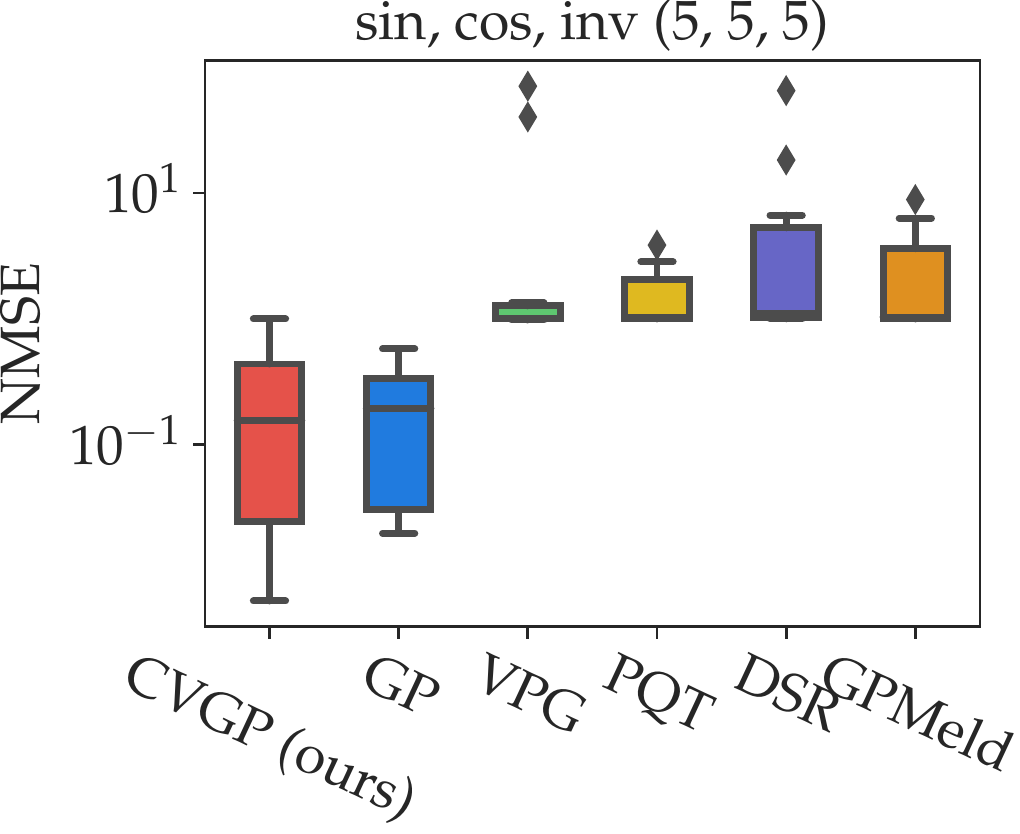}
    \includegraphics[width=0.3\linewidth]{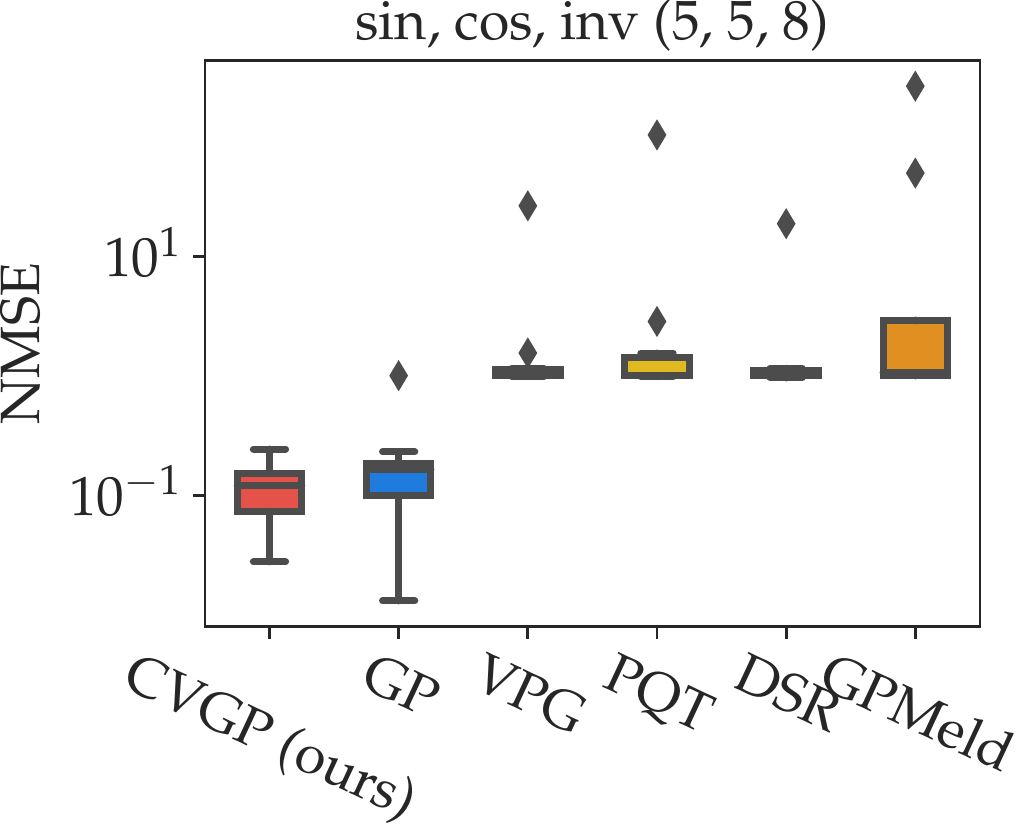}
    \includegraphics[width=0.3\linewidth]{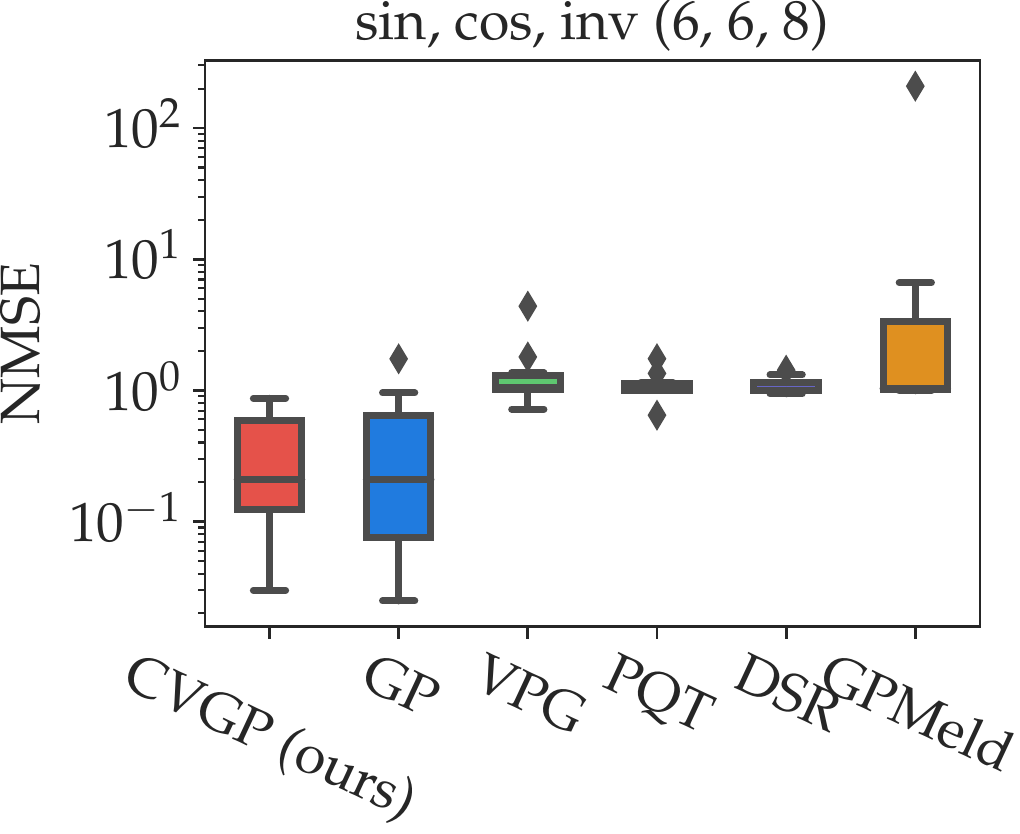}
    \includegraphics[width=0.3\linewidth]{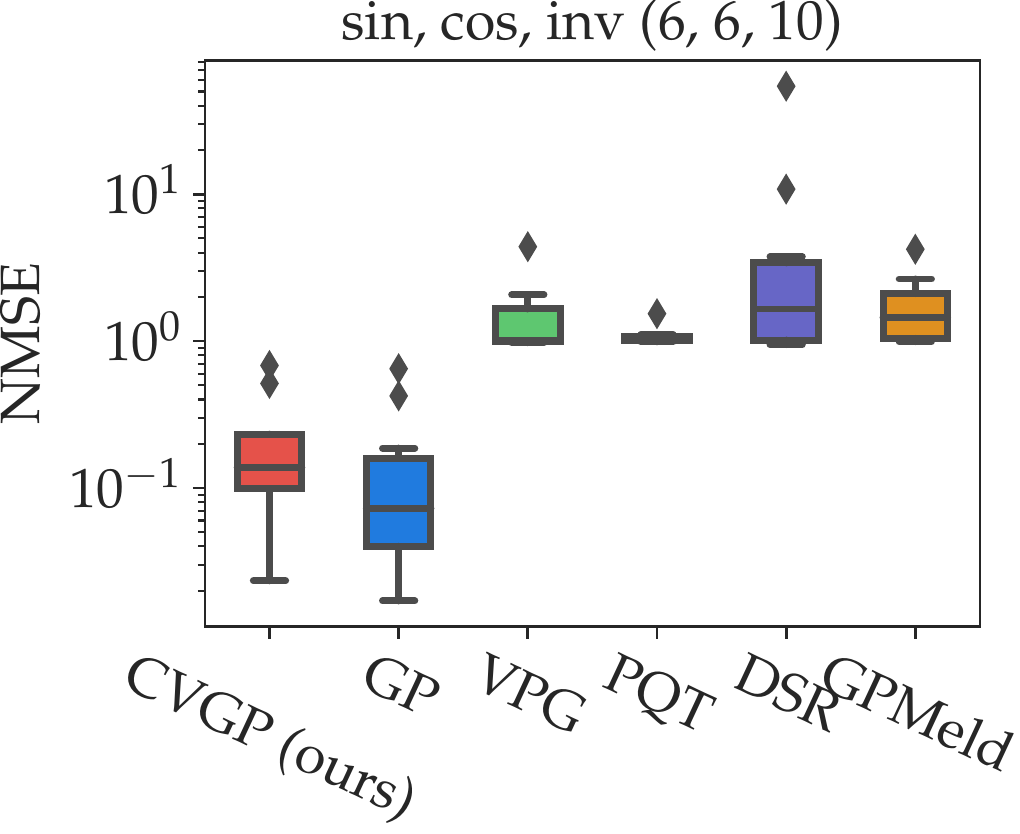}
    \caption{Quartiles of NMSE values of all the methods over several \textit{noisy} datasets. Our \method shows a consistent improvement over all the baselines considered, among all the datasets. }
    \label{fig:Quartile-nmse-noisy-full}
\end{figure}

\noindent\textbf{Quantile of NMSE Metric.} In Figure~\ref{fig:Quartile-nmse-noiseless-full}, we show the five-number summary for the NMSE metric in Table~\ref{tab:summary-nmse}, \textit{i.e.}, the minimum, 25\% quartile, median, 75\% quartile, and maximum. In Figure~\ref{fig:Quartile-nmse-noisy-full}, we show the five-number summary for the NMSE metric in Table~\ref{tab:summary-noisy}. For the datasets considered, our \method shows a consistent improvement over all the baselines.

\end{document}